\documentclass[10pt]{article} %
\usepackage[accepted]{tmlr}

\let\ANDAuthor\AND

\usepackage[utf8]{inputenc} %
\usepackage[T1]{fontenc}    %
\usepackage{hyperref}       %
\usepackage{url}            %
\usepackage{booktabs}       %
\usepackage{amsfonts}       %
\usepackage{nicefrac}       %
\usepackage{microtype}      %
\usepackage{xcolor}         %
\usepackage{amsmath}
\usepackage{amssymb}
\usepackage{mathtools}
\usepackage{amsthm}
\usepackage{thmtools}
\usepackage{placeins}
\usepackage{subcaption}  %
\usepackage{booktabs} %
\usepackage[makeroom]{cancel}
\let\AND\relax
\usepackage{algorithmic,algorithm}
\usepackage{wrapfig}
\usepackage{siunitx}

\usepackage{enumitem}
\usepackage{multirow}
\usepackage{colortbl}

\newtheorem*{lemma31*}{Proposition 3.1} 

\usepackage[capitalize,noabbrev]{cleveref}

\theoremstyle{plain}
\newtheorem{theorem}{Theorem}[section]
\newtheorem{proposition}[theorem]{Proposition}
\newtheorem{lemma}[theorem]{Lemma}
\newtheorem{corollary}[theorem]{Corollary}
\theoremstyle{definition}

\theoremstyle{remark}

\newcommand{\given}{\,|\,}
\newcommand{\V}[1]{\boldsymbol{\mathbf{#1}}}
\DeclareMathOperator*{\argmax}{arg\,max}
\DeclareMathOperator*{\argmin}{arg\,min}

\usepackage[textsize=tiny]{todonotes}

\title{Bayesian Ensembling: Insights from Online Optimization and Empirical Bayes}

\author{\name Daniel Waxman \email dan@basis.ai \\
      \addr
      Stony Brook University\thanks{Work done at Stony Brook University. Current address: Basis Research Institute \texttt{<dan@basis.ai>}.}
      \ANDAuthor
      \name Fernando Llorente \email fllorente@bnl.gov \\
      \addr 
      Brookhaven National Laboratory
      \ANDAuthor
      \name Petar M. Djuri\'c \email petar.djuric@stonybrook.edu \\ 
      \addr 
      Stony Brook University}

\def\month{01}  %
\def\year{2026} %
\def\openreview{\url{https://openreview.net/forum?id=CCvVzmfBOn}} %

\begin{document}

\maketitle

\begin{abstract}
We revisit the classical problem of Bayesian ensembles and address the challenge of learning optimal combinations of Bayesian models in an online learning setting.  To this end, we reinterpret existing approaches such as Bayesian model averaging (BMA) and Bayesian stacking through a novel empirical Bayes lens, shedding new light on the limitations and pathologies of BMA. Further motivated by insights from online optimization, we propose Online Bayesian Stacking (OBS), a method that optimizes the log-score over predictive distributions to adaptively combine Bayesian models. A key contribution of our work is establishing a novel connection between OBS and portfolio selection, bridging Bayesian ensemble learning with a rich, well-studied theoretical framework that offers efficient algorithms and extensive regret analysis. We further clarify the relationship between OBS and online BMA, showing that they optimize related but distinct cost functions. Through theoretical analysis and empirical evaluation, we identify scenarios where OBS outperforms online BMA and provide principled methods and guidance on when practitioners should prefer one approach over the other.
\end{abstract}

\section{Introduction}
Combining the opinions of multiple models is a pervasive problem in the statistical sciences, with many different names, approaches, and applications. In signal processing, for example, a commonly encountered problem is one of \emph{sensor fusion}, where information reported from several different sensors must be combined to obtain the best possible estimate \citep{khaleghi2013multisensor}. In econometrics, the problem is often known as \emph{forecast combination}, spurred by the seminal work of \citet{bates1969combination}. The Bayesian literature often refers to the problem as \emph{opinion pooling}, whichreceived attention even in the early days of Bayesian statistics \citep{stone1961opinion,degroot1974reaching,genest1986combining}. While many of the above approaches originated in the combination of point estimates, the combination of probability distributions has received significant attention since \citep{koliander2022fusion}. 

In machine learning, this issue is commonly known as \emph{ensembling}, with much recent interest due to the development of diverse models, architectures, and training modalities \citep{dietterich2000ensemble}. Many ensembling approaches exist, hinging on different assumptions and asymptotic guarantees. We will restrict our attention to Bayesian machine learning, where the goal is to combine the estimates of $K$ different {\em probabilistic} models $\mathcal{M}_1, \dots, \mathcal{M}_K$. The classical approach in this setting, known as \emph{Bayesian model averaging (BMA)}, is to weigh the estimates of each model according to their marginal likelihood \citep{hoeting1999bayesian}. When the data were generated by one of the models $\mathcal{M}_1, \dots, \mathcal{M}_K$, this is the ``correct'' way to combine models from a Bayesian perspective and is optimal, in the sense of choosing the correct model with probability $1$ in the limit of infinite data \citep{yao2018using}.

More recently, \citet{yao2018using} critically examined BMA in a more typical setting where the data are not generated by any of the candidate models. They instead proposed \emph{Bayesian stacking}, which optimizes a log-score criterion; asymptotically, Bayesian stacking corresponds to choosing the optimal \emph{convex combination} of models, juxtaposed with the optimal singular model selected by BMA.
Similar discussions and proposals have previously appeared in the econometrics community for forecasting models \citep{hall2007combining,geweke2012prediction}.

In this paper, we further study the online stacking problem, introducing a novel analysis via empirical Bayes before moving to a new analysis of Bayesian stacking in the online setting. We summarize our contributions as follows:
\begin{enumerate} %
    \item We show that \emph{Online Bayesian Stacking (OBS)} corresponds exactly to the well-studied problem of (universal) online portfolio selection (OPS).  This connection allows us to reinterpret Bayesian ensemble learning through the lens of online convex optimization and to leverage efficient, theoretically grounded algorithms (e.g., Exponentiated Gradient and the Online Newton Step). %
    \item We show theoretically how one can choose between OBS and online BMA with constant regret.
    Building on this connection, we discuss how regret bounds from the OPS literature can be applied to OBS. 
    \item We introduce a simple yet compelling argument via empirical Bayes to explain why BMA collapses and how Bayesian stacking aids in avoiding this common problem. 
    \item We perform an extensive empirical analysis of OBS using ensembles of state-of-the-art models, including Gaussian processes, variational Bayesian neural networks, and stochastic volatility forecasting models estimated with sequential Monte Carlo. We show that, in all cases tested, OBS significantly outperforms online BMA (O-BMA) and dynamic model averaging (DMA), with additional computational cost that is often negligible in comparison to online training/prediction of ML models. We also show that OBS can be particularly beneficial in non-stationary environments, often outperforming any set of fixed weights. 
\end{enumerate}

The rest of this article is structured as follows: in \cref{sec:lit_review}, we review the problem of Bayesian ensembling, discussing the BMA and stacking approaches. In \cref{sec:oco}, we show how the online variant of Bayesian stacking corresponds to the well-known (universal, online) portfolio selection problem from online convex optimization, and discuss corresponding insights that show OBS's relation to O-BMA. In \cref{sec:empirical_bayes}, we exploit this connection to derive a novel comparison of BMA and Bayesian stacking from an empirical Bayes perspective. 
In \cref{sec:experiments}, we carry out a number of realistic experiments on synthesized and real data before finally concluding in \cref{sec:discussion}.

\section{Bayesian Ensembles: A Review of Existing Methods} \label{sec:lit_review}
In this section, we review the ideas of BMA and Bayesian stacking \citet{yao2018using} and discuss some related work.

\subsection{Bayesian Model Averaging}

Bayesian ensembling methods combine a set of Bayesian models for predictive inference. We will focus on methods that create linear mixtures of posterior predictive distributions: given models $\mathcal{M}_1, \dots, \mathcal{M}_K$ each mapping $\V{x} \in \mathcal{X}\subseteq \mathbb{R}^d$ to a probability (density) over $y\in\mathcal{Y}\subseteq \mathbb{R}^r$, trained on a dataset $\mathcal{D}$, the task of Bayesian ensembling is thus to find a weight vector $\mathbf{w}$ in the $(K-1)$-dimensional simplex $\mathbb{S}^K$, which induces a posterior predictive distribution
\begin{equation} \label{eq:ensemble_def}
    p^{\text{ens}}(y \given \mathbf{x}, \mathcal{D}) \coloneqq \sum_k w_k p_k(y \given \mathbf{x}, \mathcal{D}),
\end{equation}
where $p_k(y \given \mathbf{x}, \mathcal{D})$ is the posterior predictive distribution of model ${\cal M}_k$ on the dataset $\mathcal{D}$.

The most common method for combining Bayesian models is BMA \citep{hoeting1999bayesian}, which forms the weight for a linear mixture of predictive distributions by simulating a posterior probability for model $\mathcal{M}_k$ given data $\mathcal{D}$ for weighting, 
\begin{equation} \label{eq:BMA_weights}
    w_k^{\text{BMA}} \coloneq \Pr(\mathcal{M}_k \given \mathcal{D}) = \frac{p_k(\mathcal{D}) \Pr(\mathcal{M}_k)}{\sum_k p_k(\mathcal{D}) \Pr(\mathcal{M}_k)},
\end{equation}
where $p_k(\mathcal{D}) \coloneq p(\mathcal{D}|\mathcal{M}_k)$ is the evidence of ${\cal M}_k$, and $\Pr(\mathcal{M}_k)$ is the prior probability of ${\cal M}_k$. 

The appeal of BMA to a Bayesian is immediate: if $\Pr(\mathcal{M})$ has support over the data-generating distributions, then the resulting mixture \cref{eq:ensemble_def} is a straightforward application of Bayes' rule and is thus optimal in an information-theoretic sense \citep{zellner1988optimal}. 

Much of our work focuses on the \emph{online} setting, where BMA has several additional advantages. In this setting, we first obtain a new data point $\mathbf{x}_{t+1}$, with which we must make a prediction using the available data $\mathcal{D}_{t}$. After making a prediction, the value $y_{t+1}$ is revealed and the dataset $\mathcal{D}_{t+1} \coloneq \mathcal{D} \cup \{\mathbf{x}_{t+1}, y_{t+1}\}$ is updated. In this case, we can compute $w_k^{\text{BMA}}$ recursively using posterior predictive distributions, which allows us to apply \emph{exact} BMA at time $t$ using the currently-available information: denoting the weights at time $t$ as $w_{t, k}$, with $w_{0, k} = \Pr(\mathcal{M}_k)$, \eqref{eq:BMA_weights} becomes
\begin{equation} \label{eq:BMA_weights_online}
    w_{t+1,k}^{\text{BMA}} = \frac{ w_{t, k} p_k(y_{t+1} \given \V{x}_{t+1}, \mathcal{D}_t)}{\sum_k w_{t, k} p_k(y_{t+1} \given \V{x}_{t+1}, \mathcal{D}_t)}.
\end{equation}

Two main problems arise when using BMA: first, estimation via \cref{eq:BMA_weights} or \cref{eq:BMA_weights_online} requires access to predictive distributions, either through the marginal likelihood (i.e., the prior predictive) or the posterior predictive. This is often a surmountable problem, as marginal likelihoods are usually available in conjugate models, and BMA with approximated marginal likelihoods has also performed well empirically \citep{gomez2020bayesian}.

Perhaps more importantly, BMA is only optimal in terms of predictive error in the so-called ``M-closed'' setting, where the data were generated by one of the models $\mathcal{M}_1, \dots, \mathcal{M}_K$ \citep{yao2018using,minka2000bayesian}. Indeed, in the 
limit of infinite data, BMA weights concentrate on the single model that most closely resembles the data-generating process. We provide a novel empirical Bayes-based analysis of this fact in \cref{sec:empirical_bayes}.
This may result in arbitrarily poor posterior predictive accuracy of an ensemble using $\mathbf{w}^{\text{BMA}}$ relative to some other set of weights $\mathbf{w}^{*}$.
Additionally, O-BMA can numerically collapse to the ``wrong'' model, never recovering due to numerical underflow \citep{waxman2024dynamic}. 

\subsection{Bayesian Stacking}

Bayesian stacking \citep{yao2018using} (similarly explored by \citep{clyde2013bayesian,le2017bayes}) presents an alternative way to derive a weight vector $\mathbf{w}$. In particular, Bayesian stacking finds the optimal weight vector $\V{w}^*$ in the $(K-1)$-dimensional simplex $\mathbb{S}^K$ by maximizing some score $\mathcal{S}(\V{w}, \mathcal{D}')$ over a separate dataset $\mathcal{D}'$:
\begin{equation}
    \V{w}^* \coloneq \argmax_{\V{w} \in \mathbb{S}^K} \mathcal{S}(\V{w}, \mathcal{D}').
\end{equation}
Particular attention is given to the log-score; for the predictive dataset (i.e., a holdout or validation set) $\mathcal{D}' = \{(\V{x}_n, y_n)\}_{n=1}^N$\footnote{We will typically use $n=1,\dots,N$ to denote datasets that are not processed sequentially/online, and $t=1, \dots, T$ for their sequential/online counterparts.}, the corresponding optimization problem is
\begin{equation} \label{eq:bayesian_stacking_optimization_problem}
    \V{w}^* \coloneq \argmax_{\V{w} \in \mathbb{S}^K} \sum_{n=1}^N \log \sum_{k=1}^K w_k p_k(y_n \given \V{x}_n, \mathcal{D}),
\end{equation}
where $p_k(y_n \given \V{x}_n, \mathcal{D})$ is the predictive distribution of model $\mathcal{M}_k$. This can be recognized as finding the mixture of estimators with the highest predictive likelihood over $\mathcal{D}'$. This, in turn, minimizes the 
KL
divergence between the mixture $q(y \given \V{x}) = \sum_k w_k p_k(y \given \V{x}, \mathcal{D})$ and the predictive distribution, $p(y \given \V{x})$, which represents the (unknown) true generating mechanism of 
$y$ \citep{yao2018using}.

The optimization problem \eqref{eq:bayesian_stacking_optimization_problem} is convex and provides nice asymptotic guarantees, but ``wastes'' data by requiring two separate datasets $\mathcal{D}$ and $\mathcal{D}'$. \citet{yao2018using} address this by showing the score in \eqref{eq:bayesian_stacking_optimization_problem} is well-approximated by the leave-one-out (LOO) predictive density $p(y_n \given \V{x}_n, \mathcal{D}_{-n})$, where $\mathcal{D}_{-n} = \mathcal{D} \setminus \{(\V{x}_n, y_n)\}$. For many models, the LOO predictive density can be efficiently estimated with Pareto-smoothed importance sampling \citep{vehtari2024psis}, resulting in %
an efficient optimization problem with a %
single dataset $\mathcal{D}$.
The resulting mixture of estimators has empirically shown performance superior to that of other combination methods, leading to several recent applications.

\subsection{Bayesian Stacking for Time Series \& Online Bayesian Stacking}

In the published discussion to \citet{yao2018using}, \citet{ferreira2018discussion} discusses Bayesian stacking for time series, where forecasting densities naturally serve as predictive distributions, i.e., \eqref{eq:bayesian_stacking_optimization_problem} becomes
\begin{equation} \label{eq:bayesian_stacking_optimization_problem_timeseries}
    \V{w}^* \coloneq \argmax_{\V{w} \in \mathbb{S}^K} \sum_{t=t^*+1}^{T} \log \sum_{k=1}^K w_k p_k(y_{t} \given \V{x}_{t}, \mathcal{D}_{1:t-1}),
\end{equation}
where the first $t^*$ data are devoted to (pre-)training each model, and $\mathcal{D}_{1:t-1}$ denotes the first $t-1$ data points.
This has notable similarities to pooling methods described by \citet{hall2007combining,geweke2011optimal}, with the exception of the summation starting with index $t=1$, reminiscent of the LOO approach. This method was applied by \citet{geweke2012prediction}, where optimal weights were computed quarterly. 
``Windowed'' approaches for autocorrelated time series data, where only the last $T_p$ points are considered, have also been deployed \citep{jore2010combining,aastveit2014nowcasting}, but these are suboptimal in the more general case where data may be exchangeable conditioned on $\V{x}_t$. Related to windowed approaches is dynamic model averaging, which is O-BMA with forgetting factors \citep{raftery2010online}.
To the best of our knowledge, Bayesian stacking and its time series variants have not been applied in the \emph{online} setting, where weights are estimated as new data become available. We will refer to approaches to the online problem as \emph{online Bayesian stacking (OBS)}.

As we will see, OBS is a special case of online convex optimization (OCO) \citep{hazan2022introduction}, where ``learning from experts'' is well-studied. We show that the popular Hedge algorithm \citep{freund1997decision} generalizes BMA with a learning rate, but optimizes a different loss function from OBS. Our proposed OBS is differentiated from existing approaches in several ways. First, it directly emulates Bayesian stacking, of recent interest to the Bayesian community. Second, our connection to OCO yields an extremely efficient implementation, unlike methods based on data-driven portfolio selection (e.g., \citet{bacsturk2019forecast}) that rely on particle filtering. Finally, our approach is more general for online learning, whereas data-driven strategies rely on the autocorrelation of time series data.

\section{Insights from Online Convex Optimization} \label{sec:oco}
As Bayesian stacking is natively framed as an optimization problem, it is natural to study OBS from the optimization perspective as well. It turns out that our studies here are fruitful: by interpreting posterior predictive values as asset prices, OBS corresponds exactly to a classical problem of portfolio selection. An OCO perspective additionally rediscovers a development of \citet{vovk2001competitive}, which shows that BMA is exactly the Hedge algorithm with a specific choice of learning rate. We additionally discuss a hybrid approach meant to detect whether the model class is M-open or M-closed (or, more precisely, if OBS mixtures will outperform expert selection).

\subsection{The Portfolio Selection Problem}

 We will first review the basic ideas of online portfolio selection (OPS), but our discussion will necessarily be brief; the interested reader is referred to the recent survey of \citet{li2014online}, or the excellent lecture notes of \citet{hazan2022introduction,orabona2019modern}. 

 \textbf{Problem Statement} Consider an investment manager overseeing $K$ stocks, aiming to maximize total future wealth. They create the \emph{best constantly rebalanced portfolio (BCRP)}, i.e., at each time step they re-allocate current wealth such that $w_k$ is in stock $k$ (classic portfolio selection assumes no transaction fees). Assets are allocated based on \emph{price-relative vectors} $\V{r}_t$, where $r_{t, k}$ is stock $k$'s relative value increase from day $t-1$ to $t$. The goal is to maximize the (multiplicative) wealth at time $T$, or more conveniently, the (additive) log-wealth:
\begin{equation} \label{eq:log_wealth}
    \text{wealth}_T \coloneq \prod_{t=1}^{T} \V{w}^\top \V{r}_t \leftrightsquigarrow  \log \text{wealth}_T \coloneq \sum_{t=1}^{T}\log  (\V{w}^\top \V{r}_t).
\end{equation}

In the \emph{online} version of the portfolio selection problem, we seek to find $\V{w}_t$ that minimizes the regret with respect to the optimal weights $\V{w}^*$, 
\begin{equation} \label{eq:portfolio_regret}
    R_T \coloneq \sum_{t=1}^{T}\log  ({\V{w}^{*}}^\top \V{r}_t) - \sum_{t=1}^{T}\log  (\V{w}_t^\top \V{r}_t),
\end{equation}
where the weights $\V{w}_t$ are determined after observing $\V{r}_{t-1}$. 

Although many approaches explicitly model returns as stochastic processes, others allow $r_{t, k}$ to be chosen arbitrarily, and indeed by an adversary. Algorithms that achieve sublinear regret with respect to any non-negative sequence of returns $r_{t, k}$ are known as \emph{universal} \citep[p. 144]{orabona2019modern}. We draw an analogy between the price-relative vectors $\V{r}_{t}$ and the predictive likelihoods $p(y_t \given \V{x}_t, \mathcal{D}_{t-1})$,  which make the adversarial case of particular interest to us. 

\textbf{Algorithms} The portfolio selection problem has an attractive structure from the optimization perspective: the log wealth is a concave function defined over the simplex, a convex set. OPS therefore falls under the (online) convex optimization (OCO) umbrella \citep{hazan2022introduction,orabona2019modern}. To keep with standard terminology, we will introduce each method as minimizing the convex loss function given by $-\log \text{wealth}_T$, which is equivalent to maximizing the concave function $\log \text{wealth}_T$.

The simplest OCO algorithm is online (sub)gradient descent (OGD), a straightforward extension of the classical gradient descent algorithm to a sequence of losses $\ell_1, \dots, \ell_T$, each a function of some generic quantity $\V\theta$ belonging to a convex set $\mathcal{K}$. In our case, $\V{\theta}_t$ represents weights $\V{w}_t$, $\ell_t$ is the negative log-return, $-\log(\V{w}_t^\top \V{r}_t)$, and $\mathcal{K}$ is the simplex $\mathbb{S}^K$.
At each time instant, $\V{\theta}_t$ is updated in the direction of the gradient $\nabla \ell_t$ and then projected back onto the convex set $\mathcal{K}$ (e.g., \citep[Algorithm 2.3]{hazan2022introduction}). 

In many problems, it turns out to be beneficial to consider regularizers that account for the geometry of $\mathcal{K}$ explicitly, resulting in different update steps. This is an online form of the classical mirror descent \citep{nemirovsky1983wiley}.
When $\mathcal{K}$ is the simplex $\mathbb{S}^{K}$ (as in OPS), one popular choice is entropic regularization (e.g., \citep[Section 5.4.2]{hazan2022introduction}). The resulting algorithm is known as \emph{Exponentiated Gradients (EG)} \citep{helmbold1998line},
which uses the update step
\begin{equation} \label{eq:exponentiated_gradients}
    \V{w}_{t+1} = \frac{\V{w}_t \odot \exp(-\eta \nabla_{\V{w}} \ell_t)}{\lVert \V{w}_t \odot \exp(-\eta \nabla_{\V{w}} \ell_t) \rVert_1},
\end{equation}
where $\odot$ is the elementwise product, and $\eta$ is a learning rate. For the cost function $\ell_t(\V{w}) = -\log \V{w}_t^\top \V{r}_t$, we have $\nabla_{\V{w}} \ell_t = -\V{r}_t / \V{w}_t^\top \V{r}_t$, and therefore, 
\begin{equation} \label{eq:exponentiated_gradients_portfolio}
    \V{w}_{t+1} = \frac{\V{w}_t \odot \exp(\eta \V{r}_t / \V{w}_t^\top \V{r}_t)}{\lVert \V{w}_t \odot \exp(\eta \V{r}_t / \V{w}_t^\top \V{r}_t) \rVert_1}.
\end{equation}
The EG algorithm is conceptually simple, with relatively good regret bounds whenever the gradient is bounded. However, more modern algorithms, such as the \emph{online Newton step (ONS)} \citep{hazan2007logarithmic}, can provide tighter bounds on the regret. 

Of note for applications in non-stationary environments are OCO algorithms designed for such environments. In our experiments, for example, we will utilize \emph{discounted ONS (D-ONS)} \citep{yuan2020trading}, an ONS variant that includes a forgetting factor over the second-order information.

\paragraph{Market Variability and Soft-Bayes} In order to achieve optimal regret bounds with EG or ONS, we must assume that the ratio $\alpha$ of the minimum return to maximum return at any time $t$ --- called the \emph{market variability parameter} --- is bounded. In this case, $\ell_t$ is $1$-exp-concave with bounded gradients, meaning that $\exp(-\ell_t) = \V{w}_t^\top \V{r}_t$ is a concave function \citep{hazan2022introduction}. So long as this assumption holds, the ONS algorithm provides tighter regret bounds than EG. 

Moving beyond the assumption of a bounded $\alpha$ is the more recent Soft-Bayes \citep{orseau2017soft}. Soft-Bayes proposes weight updates using a learning rate ${\eta}\in(0, 1)$ as 
\begin{equation} \label{eq:softbayes}
    \mathbf{w}_{t+1} = \mathbf{w}_t \odot \left( 1 - \eta + \eta \frac{\mathbf{r}_t}{\mathbf{w}_t^\top \mathbf{r}_t}\right).
\end{equation}
\citet{orseau2017soft} provide an interpretation of Soft-Bayes in terms of ``slowing down'' O-BMA, and the resulting algorithm provides state-of-the-art regret bounds without the assumption of bounded gradients. We remark that while Soft-Bayes is in some aspects similar to our work, their focus is on developing algorithms and theoretical bounds for log-loss mixtures of experts. In the current work, we focus on statistical insights, connecting this online log-loss problem to the recently popular Bayesian stacking, and showing empirically that OBS is viable for modern Bayesian machine learning.

\subsection{Online Bayesian Stacking is Portfolio Selection}
Our core insight in this section is that the utility of \eqref{eq:bayesian_stacking_optimization_problem_timeseries} becomes \eqref{eq:log_wealth} when the market gain $r_{t, k}$ is defined by the predictive density $p_k(y_{t} \given \V{x}_t, \mathcal{D}_{t-1})$. Indeed, our only requirement in a universal portfolio algorithm is that $r_{t, k}$ is nonzero, which is a rather mild and near-universally satisfied assumption for the predictive distribution of a regression model. Furthermore, using a constant rebalanced portfolio in the regret \eqref{eq:portfolio_regret} is appropriate, as a constant rebalanced portfolio corresponds to the constant weights used in (offline) Bayesian stacking.

A point of nuance in applying OPS algorithms is the existence of a market variability parameter $\alpha$, requiring pre-determined maximum and minimum predictive densities. The maximum is clear for proper Bayesian models, and a minimum may be assumed (e.g., via compact data spaces or bounded model error). However, $\alpha$ may still be very small, which can be problematic for EG/ONS regret analysis. Bounding regret using subgradient norms can produce tighter guarantees, since non-adversarial predictions from failing models (i.e., $w_{t, k} \approx 0$) are unlikely to suddenly dominate. If such outliers are a concern, Soft-Bayes also provides $\alpha$-independent regret bounds.

\begin{algorithm}[t]
\caption{Online Ensemble Update: \textbf{\textcolor[HTML]{cd5c5c}{OBS with EG}} or \textbf{\textcolor[HTML]{83B692}{OBS with Soft-Bayes}} or \textbf{\textcolor[HTML]{4e92f2}{O-BMA}}.} \label{alg:obs_vs_obma}
\begin{algorithmic}[1]
\STATE \textbf{Input:} Data stream $\{(\mathbf{x}_t,y_t)\}_{t=1}^T$, models $\{M_k\}_{k=1}^K$, initial weights $\mathbf{w}_0\in\Delta^K$, learning rate $\eta$.
\FOR{$t=1$ \TO $T$}
    \item[] \textbf{\textit{\textcolor{gray}{// Prediction Step}}}
    \STATE Receive $\mathbf{x}_t$ and compute each model's predictive density $p_k(y \mid \mathbf{x}_t, \mathcal{D}_{t-1})$ for $k = 1, \dots, K$ 
    \STATE Form ensemble prediction via \cref{eq:ensemble_def}
    \STATE Output prediction and receive true label $y_t$.
    \item[] \textbf{\textit{\textcolor{gray}{// Update Step}}}
    \STATE Update model weights using \textbf{\textcolor[HTML]{cd5c5c}{OBS w/ EG \eqref{eq:exponentiated_gradients_portfolio}}} or \textbf{\textcolor[HTML]{83B692}{OBS w/ Soft-Bayes \eqref{eq:softbayes}}} or \textbf{ \textcolor[HTML]{4e92f2}{O-BMA \eqref{eq:BMA_weights_online}}} 
    \STATE Update dataset:
      $\mathcal{D}_t \gets \mathcal{D}_{t-1}\cup\{(\mathbf{x}_t,y_t)\}.$ 
\ENDFOR
\end{algorithmic}
\end{algorithm}

We provide pseudocode for OBS and O-BMA in \cref{alg:obs_vs_obma}. The pseudocode underscores the similarity between OBS and O-BMA, replacing a single update equation. Other algorithms for OPS may be used by changing Line 6.

\subsection{Online Bayesian Model Averaging is the Hedge Algorithm with Learning Rate $1$}
Now that we have established an equivalence between OBS and the portfolio selection problem, one may wonder if a similar connection holds for O-BMA. Indeed, it is classically known that O-BMA updates are a specific choice of the \emph{Hedge algorithm} \citep{freund1997decision}, which performs expert selection. While this connection is established \citep{vovk2001competitive}, we rederive it as a useful narrative tool to motivate O-BMA as a model selection algorithm, which we show clearly using \cref{eq:bma_lemma}.

Close inspection reveals that $w_{t+1, k}^{\text{BMA}}$ is precisely recovered by EG when the learning rate is $\eta = 1$ and the loss is
\begin{equation}
    \ell_t(\V{w}_t) = -\sum_k w_{t, k} \log p_k(y_t \given \V{x}_t, \mathcal{D}_{t-1}),
\end{equation}
whose gradient is  $\nabla_{w_{t, k}}\ell_t(\V{w_t}) = -\log p_k(y_t|\V{x}_t,\mathcal{D}_{t-1})$.
This suggests that O-BMA may act similarly to an OCO algorithm that minimizes regret with respect to a different loss function,
\begin{equation} \label{eq:bma_loss}
    \mathcal{L} = -\sum_{t=1}^T \sum_{k=1}^K w_k \log p_k(y_t \given \V{x}_t, \mathcal{D}_{t-1}).
\end{equation}
This reveals a key insight: OBS aims for the best post-hoc expert \emph{mixture}, while O-BMA targets the best \emph{single} expert.
The corresponding bound achieves constant regret; in fact, in the typical proof of this regret upper bound, we prove an equality regarding the regret. The following result is therefore known, but is presented in a somewhat unorthodox way to allow us to show \emph{lower} bounds on the regret.

\begin{proposition} \label{eq:bma_lemma}
    Let the regret of the BMA mixture with respect to the best individual model be defined as
    $
        R_T = \sum_t \log p_{k^{*}}(y_{t} \given \V{x}_t, \mathcal{D}_{t-1}) -\sum_t \log\left(\sum_k w_{t, k} p_k(y_t \given \V{x}_t, \mathcal{D}_{t-1}) \right),$
    where $k_*$ is the model with the largest marginal likelihood. Then $R_T$ is related to an evidence gain in $M_{k^{*}}$,
    \begin{equation}
        R_T = \log \Pr(\mathcal{M}_{k^*} \given \mathcal{D}_T) / \Pr(\mathcal{M}_{k^*}).
    \end{equation}
\end{proposition}
Thus, if $\Pr(\mathcal{M}_{k^*} \given \mathcal{D})$ is bounded below, so is the regret. In particular, \cref{eq:bma_lemma} applied the following theorem directly, which asserts O-BMA acts as an ``optimizer'' to \cref{eq:bma_loss}.

\begin{corollary}
    Let the regret of the BMA mixture with respect to the best individual model be defined as in \cref{eq:bma_lemma}. If the posterior probability of the optimal model $\mathcal{M}_{k^*}$ exceeds its prior probability, i.e., $\Pr(\mathcal{M}_{k^*} \given \mathcal{D}) \geq \Pr(\mathcal{M}_{k^*})$, then $R_T$ is bounded both above and below, 
    \begin{equation} \label{eq:bma_regret_bound}
        0 \leq R_T \leq -\log \Pr(\mathcal{M}_{k^{*}}).
    \end{equation}
\end{corollary}

Thus, under typical scenarios (e.g., uniform priors on $\mathcal{M}_k$), \cref{eq:bma_regret_bound} becomes fairly tight, and the solutions become strong online optimizers of \cref{eq:bma_regret_bound}.

\subsection{Regret Analysis}

Connecting OBS with OPS makes regret bounds from the OCO literature available. Different choices of algorithms and assumptions on the predictive densities $\V{p}_t$ affect the obtainable regret analysis (c.f. \citet[Table 1]{van2020open}); however, without additional assumptions on $\V{p}_t$, the most efficient algorithms typically obtain regret of order $\mathcal{O}(\sqrt{T})$ (such as Soft-Bayes \citep{orseau2017soft}). If a bound on the norm of the gradients may be assumed, ONS can provide regret on the order $\mathcal{O}(\log T)$ (with runtime scaling with $K^2$). Variants of simple OCO algorithms may even provide regret with respect to time-varying optimal parameters, such as the D-ONS algorithm \citep{yuan2020trading}. 

We further discuss regret bounds in \cref{app:regret_analysis}. In \cref{app:gp_regret_example}, we provide an example that shows that, with mild additional assumptions, we can even recover regret bounds from O-BMA applications, albeit with potentially worse constants - this is despite the fact that O-BMA cannot provide comparable bounds when regret is measured against the mixture loss.

\subsection{A Hybrid Approach}

In the case where the M-closed scenario is plausible, but this fact is unknown, OBS still achieves vanishing regret with respect to the singular best expert: this follows because the problem of determining an optimal mixture of densities is harder than the corresponding expert selection problem, which is included as a special case. 

Nevertheless, as determined in \cref{eq:bma_lemma}, the corresponding regret guarantees for O-BMA in the M-closed setting are extremely strong, achieving constant regret as an upper bound. When a model extremely close to the data generating process is available, OBS is thus potentially suboptimal.

One potential approach to ameliorate this is to consider a hybrid method, where O-BMA and OBS mixtures are maintained and further averaged via a secondary layer of O-BMA. That is, we consider two sets of weights, $\V{w}_t^{\text{O-BMA}}$ and $\V{w}_t^{\text{OBS}}$, along with a secondary set of weights $\V{v}_t \in \mathbb{S}^2$. The corresponding predictive distribution is thus
\begin{equation} \label{eq:hybrid}
    p^\text{hybrid}(y_t \given \V{x}_t, \mathcal{D}_{1:t-1}) = v_1 \underbrace{\left(\sum_k w_{t, k}^{\text{O-BMA}} p_k(y_t \given \V{x}_t, \mathcal{D}_{t-1})\right)}_{p^{\text{O-BMA}}(y_t \given \V{x}_t, \mathcal{D}_{t-1})} + v_2 \underbrace{\left(\sum_k w_{t, k}^{\text{OBS}} p_k(y_t \given \V{x}_t, \mathcal{D}_{t-1}) \right)}_{p^{\text{OBS}(y_t \given \V{x}_t, \mathcal{D}_{t-1})}}.
\end{equation}
The predictive performance of this mixture achieves constant regret with respect to the best model, once again, while maintaining the expressiveness of the corresponding OBS solution. In particular, we may obtain the following bound:
\begin{theorem}
    Let the regret of the hybrid mixture in \cref{eq:hybrid} for the best individual model be defined as in \cref{eq:bma_lemma}, i.e., $R_T = \sum_t \log p_{k^*}(y_t \given \V{x}_t, \mathcal{D}_{t-1}) - \sum_t \log p^{\text{hybrid}}(y_t \given \V{x}_t, \mathcal{D}_{t-1})$, where $k^*$ is the model with the largest marginal likelihood. Further assume uniform prior weights over $\V{v}$ and $\V{w}^{\text{O-BMA}}$. Then $R_T$ may be bounded as
    \begin{equation} 
        R_T \leq \log K + \log 2.
    \end{equation}
\end{theorem}
\begin{proof}
    From the ``outermost'' perspective of ensembling (i.e., with $\V{v}$), we have by \cref{eq:bma_lemma} a regret bound with respect to any ``outer'' expert (i.e., the O-BMA or OBS mixtures):
    \begin{equation} \label{eq:hybrid_eq_1}
        \sum_t \log p^{\text{O-BMA}}(y_t \given \V{x}_t, \mathcal{D}_{t-1}) - \sum_t \log p^{\text{hybrid}}(y_t \given \V{x}_t, \mathcal{D}_{t-1}) \leq \log 2.
    \end{equation}
    Further, from \cref{eq:bma_lemma}, O-BMA achieves a bounded regret
    \begin{equation} \label{eq:hybrid_eq_2}
        \sum_t \log p_{k^*}(y_t \given \V{x}_t, \mathcal{D}_{t-1}) - \sum_t \log p^{\text{O-BMA}}(y_t \given \V{x}_t, \mathcal{D}_{t-1}) \leq \log K.
    \end{equation}
    Combining \cref{eq:hybrid_eq_1} and \cref{eq:hybrid_eq_2} completes the theorem.
\end{proof}

We note by a symmetric argument that if the OBS mixture has a clearly stronger predictive likelihood, the ``outer'' level of O-BMA will select the OBS mixture as the best expert, similarly incurring a simple additive regret factor.

\section{An Empirical Bayes Perspective} \label{sec:empirical_bayes}

An empirical Bayes (EB) perspective \citep{robbins1964empirical} elucidates BMA and Bayesian stacking properties, offering direct arguments for their limitations. Recall that in EB, hyperparameters $\V{\psi}$ are found by maximizing the marginal likelihood, rather than being marginalized out. We show that O-BMA and OBS have subtly distinct EB justifications. While we focus notationally on the online setting, similar results hold for batch processing.

Let $p(y_1|\V{\psi})$ be the prior predictive density of $y_1$ (model parameters are integrated out), and $p(y_i|\V{y}_{1:i-1},\V{\psi})$ for $i=2,\dots,t$ are posterior predictive densities.
Then the \emph{prequential principle} \citep{dawid1984present} studies a model through a predictive decomposition, 
choosing $\V{\psi}$ via $ p(\V{y}_{1:t}|\V{\psi}) = \prod_{\tau}  p(y_{\tau}|\V{y}_{1:{\tau}-1},\V{\psi}).$ Further selecting $\V\psi$ through the optimization problem
\begin{align} \label{eq:eb_estimator} %
    \V{\psi}_\star = \arg\max_{\V{\psi}} \sum_{i=1}^t\log p(y_i|\V{y}_{1:i-1},\V{\psi}) 
\end{align}
is termed \emph{empirical Bayes} or type-II maximum likelihood estimate.

\subsection{BMA is Empirical Bayes over an Indicator Variable}

It is well-known in the literature that BMA collapses to the model with the highest marginal likelihood when the amount of data increases \citep{yao2018using}, but explanations of this fact are not often presented. Using our insights that O-BMA is the Hedge algorithm, i.e., that it minimizes regret with the loss function \eqref{eq:bma_loss}, this fact is straightforward to see: $w_k$ should be $1$ for the model with the highest marginal likelihood and $0$ for the others.
In this sense, we may interpret BMA as performing empirical Bayes over an indicator variable $\V{z}$ in the corresponding mixture distribution. To see this, we note that \eqref{eq:bma_loss} can be written as
\begin{align}\label{eq:aqui}
    \mathcal{L} 
     &=
     \sum_{k=1}^K w_k \log \prod_{t=1}^T p_k(y_t \given \V{x}_t, \mathcal{D}_{t-1}) = \sum_{k=1}^K w_k\log  p_k( \mathcal{D}_{T}),
\end{align}
and that 
    $\sum_{k=1}^K w_k\log  p_k( \mathcal{D}_{T})\leq \max_k \log  p_k(\mathcal{D}_{T}).$
Hence, the optimal weight vector corresponds to $w_k = 0$ for all $k$ except for the model $k^*=\argmax_k p_k(\mathcal{D}_T)$ with the highest marginal likelihood, for which $w_{k^*}=1$. 
Note also that by Jensen's inequality, \eqref{eq:aqui} is a lower bound to $\log(\sum_kw_kp_k(\mathcal{D}_T))$. If we set $w_k=\text{Pr}(\mathcal{M}_k)$, this is the log of the evidence of the BMA model, so we can interpret BMA as performing empirical Bayes on the discrete prior probabilities of the models.

\subsection{OBS is Empirical Bayes Estimation over Mixture Weights}

On the other hand, OBS can be seen as performing empirical Bayes on the mixture weights themselves. To see this, we interpret $\V\psi$ as $\V{w}$, and the posterior predictive density
\begin{align}\label{eq_likelihood_stacking}
    p(y_i|\V{y}_{1:i-1},\V{w}) = \sum_{k}w_k p_k(y_i|\V{y}_{1:i-1}),
\end{align}
is the weighted mixture of the predictive densities of the individual models. The objective function that is being maximized in stacking in Eq. \eqref{eq:bayesian_stacking_optimization_problem_timeseries}
is
\begin{align}
    \sum_{i=1}^t\log\left(\sum_{k}w_k p_k(y_i|\V{y}_{1:i-1})\right) &= \sum_{i=1}^t\log p(y_i|\V{y}_{1:i-1},\V{w}) = \log p_\text{stack}(\mathcal{D}_t|\V{w}).
\end{align}
Hence, in stacking, we obtain the weights by maximizing the log-evidence of the `stacking' model, $\log p_\text{stack}(\mathcal{D}_t|\V{w})$, where the weights correspond to hyperparameters. Note the difference between this marginal likelihood and the one in BMA, where the marginal likelihood is itself a mixture and each component is independent of $\V{w}$.

\section{Experiments} \label{sec:experiments}
Thus far, we have primarily focused on the theoretical properties of OBS. In this section, we provide empirical evidence that OBS can be beneficial. We consider four main scenarios: an illustrative toy example (\cref{sec:exp_subset}), variational Bayesian neural networks (\cref{sec:exp_bnn}), SMC-based stochastic volatility models (\cref{sec:exp_forecasting}), and dynamic Gaussian processes in non-stationary environments (\cref{sec:online_nonstationary}).
We provide details on our experimental setup, baselines, and code in \cref{sec:experimental_setup_and_code}. 

\begin{figure}[t]
     \centering
     \begin{subfigure}[b]{0.3\textwidth}
         \centering
         \includegraphics[width=\textwidth]{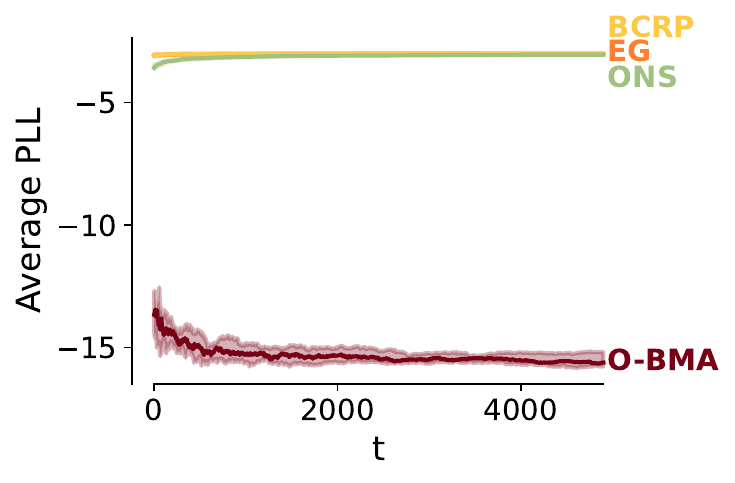}
         \caption{Toy Example: Open (\cref{sec:exp_subset})}
         \label{sf:open}
     \end{subfigure}
     \begin{subfigure}[b]{0.3\textwidth}
         \centering
         \includegraphics[width=\textwidth]{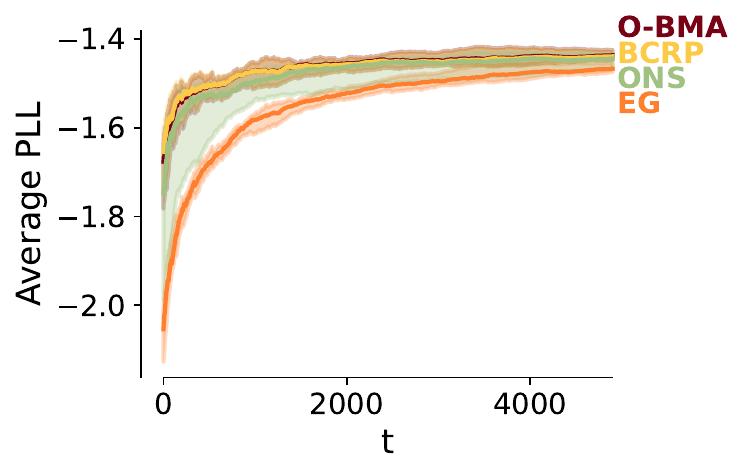}
         \caption{Toy Example: Closed (\cref{sec:exp_subset})}
         \label{sf:closed}
     \end{subfigure}
        \caption{The average predictive log-likelihood (higher is better) in the toy example.
        `EG'' is exponentiated gradients, ``ONS'' is the Online Newton Step, ``BCRP'' is the optimal constant rebalanced portfolio (offline baseline), and ``O-BMA'' is O-BMA. 
        Lines denote the median and shaded area represent the 10th to 90th percentiles over 10 trials. The first 100 samples are suppressed for readability.}
        \label{fig:exp_results}
\end{figure}

\subsection{Subset Linear Regression} \label{sec:exp_subset}

We revisit a classical problem from \citet{breiman1996stacked}, also used by \citet{yao2018using}, generating i.i.d. data from a 15-dimensional Gaussian linear model with weak predictors, as in \citet[Section 4.2]{yao2018using}. We consider two scenarios with an ensemble of 15 regression models. In the \textbf{``Open'' setting}, the ensemble consists of 15 univariate models ($y \approx \theta_k x_k$), none of which match the true process, where performance depends on combining models. In the \textbf{``Closed'' setting}, model $k$ uses features $x_1, \dots, x_k$, making model 15 the true model. In both scenarios, hyperparameters were pre-trained using empirical Bayes on 1000 points, followed by online deployment of OBS or O-BMA for 5000 points. Further experimental details are in \cref{app:more_linear_regressors}. As shown in \cref{sf:open,sf:closed}, OBS variants (EG, ONS) dramatically outperform O-BMA in the ``open'' setting. In the ``closed'' setting, OBS performs nearly as well as the theoretically optimal O-BMA, with ONS slightly outperforming EG.

As detailed in the theoretical section, one reason we may expect O-BMA to perform worse than OBS is the weight collapse of BMA. We can empirically validate this property in our toy experiment, with the expectation that O-BMA collapses to a single weight, and OBS retains a proper ``mix'' of models. 

In \cref{sf:weights_final_open}, we show the final weights in the ``open'' subset linear regression experiment for O-BMA, OBS, and the BCRP (i.e., the optimal pooling of \citet{geweke2011optimal}, or Bayesian stacking from the prequential principle \citep{yao2018using}). In \cref{sf:weights_evo_open}, we show the evolution of the BMA and OBS weights as new data arrive. Finally, in \cref{sf:weights_final_closed} and \cref{sf:weights_evo_closed}, we show the analogous final weights and the evolution of weights for the ``closed'' subset linear regression experiment.

These results show promising evidence of our approach: OBS converges to a set of weights very similar to the best retrospective weights within {5000} samples in both the ``open'' and ``closed'' variations. 
OBS with ONS seems to exhibit more of the ``collapsing'' than OBS with EG, but these differences could be due to the hyperparameter choices, which were set to the values, $\eta = 10^{-2}$ for EG, and $\delta = 0.8$, $\eta=\beta=10^{-2}$ for ONS.
We further observe the collapse of O-BMA, and that OBS can also ``collapse'' if appropriate.

\begin{figure*}
    \centering
    \begin{subfigure}[b]{0.5\textwidth}
         \centering
         \includegraphics[width=\textwidth]{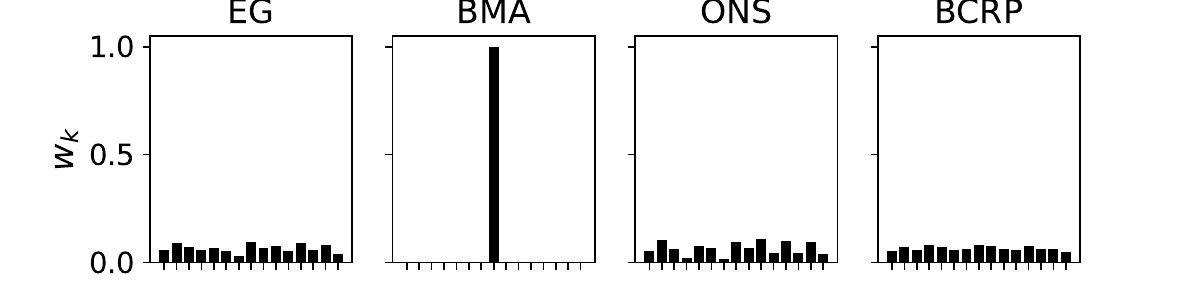}
         \caption{Toy Example: Open}
         \label{sf:weights_final_open}
     \end{subfigure}
     \begin{subfigure}[b]{0.45\textwidth}
         \centering
         \includegraphics[width=\textwidth]{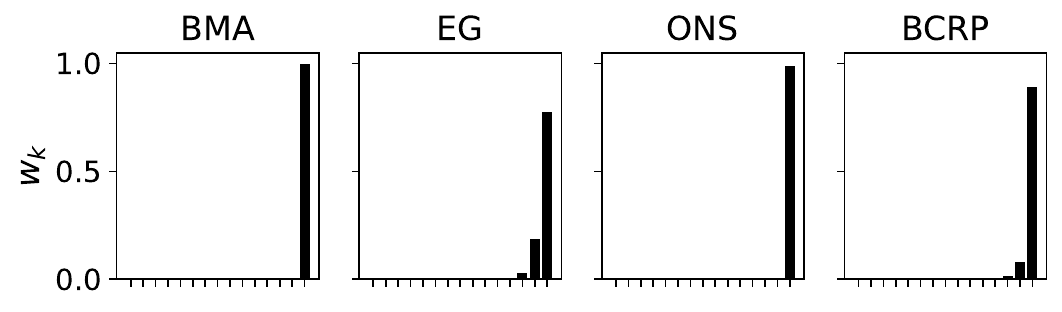}
         \caption{Toy Example: Closed}
         \label{sf:weights_final_closed}
     \end{subfigure}
    \caption{The final weights in the ``open'' and ``closed'' subset linear regression experiment.}
    \label{fig:toy_example_final_weights}
\end{figure*}

\begin{figure*}
    \centering
    \begin{subfigure}[b]{0.45\textwidth}
         \centering
         \includegraphics[width=\textwidth]{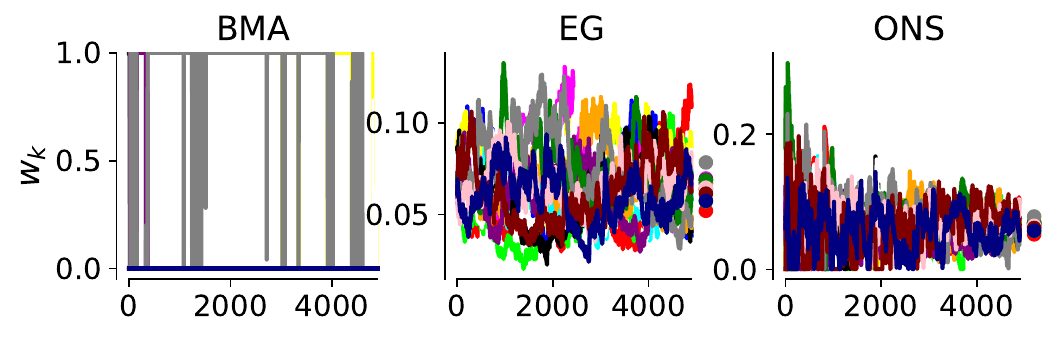}
         \caption{Toy Example: Open}
         \label{sf:weights_evo_open}
     \end{subfigure}
     \begin{subfigure}[b]{0.45\textwidth}
         \centering
         \includegraphics[width=\textwidth]{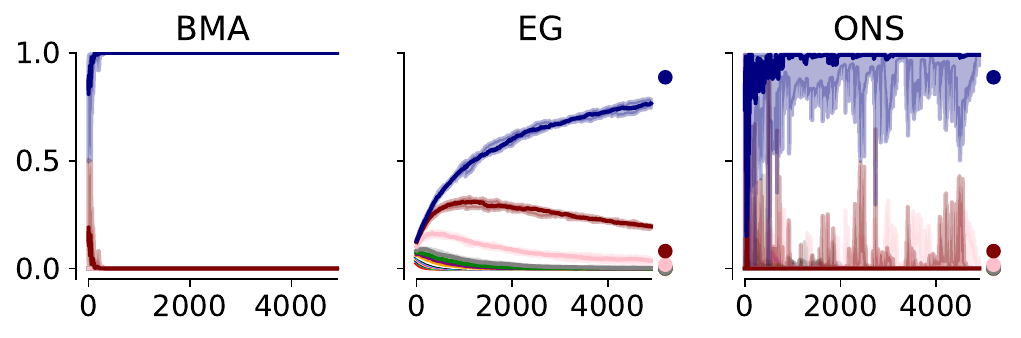}
         \caption{Toy Example: Closed}
         \label{sf:weights_evo_closed}
     \end{subfigure}
    \caption{The evolution of the weight vector $\V{w}_t$ as a function of $t$ in the ``open'' and ``closed'' subset linear regression experiment. Results are shown for a single trial due to the noisy nature of the plots. Dots on the right side of a plot denote the final weights of the BCRP.}
    \label{fig:toy_example_weights_evo}
\end{figure*}

\subsection{Online Variational Inference} \label{sec:exp_bnn}
We now move to a more practical application in Bayesian machine learning, applying OBS to online Bayesian neural networks. 
We use the recently proposed Bayesian online natural gradient (BONG), which optimizes the expected log-likelihood with online mirror descent \citep{jones2024bayesian}. Many different variants of BONG and related approaches are tested, unifying them under a common framework; we make online predictions on MNIST \citep{lecun2010mnist} using the best variant tested, ``\texttt{bong-dlr10-ef\_lin}''. To create ensembles, we train five Bayesian feedforward neural networks with different initializations. We provide more details on the experimental setup in \cref{app:more_mnist}.

We visualize the resulting predictive log-likelihoods in \cref{sf:BONG_NLL}. Although less dramatic than in the toy example, we observe clear improvements in predictive performance using OBS instead of O-BMA and DMA. Furthermore, the specific OCO algorithm used does not seem to matter much, with Soft-Bayes, EG, and ONS all resulting in solutions similar to the BCRP. In \cref{app:more_mnist}, we additionally visualize the evolution and final weights, which show that BMA collapses to a single model, whereas OBS converges to the optimal, more balanced weights.

\subsection{Online Forecasting} \label{sec:exp_forecasting}

We compare BMA and OBS variants for forecasting S\&P data, which were used in an offline analysis by \citet{geweke2011optimal,geweke2012prediction}, ensembling diverse GARCH models. Model parameters are estimated online using sequential Monte Carlo \citep{doucet2001introduction}, where the typical ``bank of filters'' approach is equivalent to BMA. Eight models (2 per class, with priors sampled from uniform hyperpriors) were generated -- see \cref{app:more_garch} for further details.

\begin{wraptable}{r}{0.3\textwidth}
  \centering
  \begin{tabular}{lc}
    \toprule
    Method & Median Reward \\
    \midrule
    EG         & 3.46 \\
    BMA        & 3.44 \\
    ONS        & 3.46 \\
    Soft-Bayes & 3.31 \\
    \bottomrule
  \end{tabular}
  \caption{Median reward (predictive log-likelihood; higher is better) at the final time step in the GARCH experiment with $100$ models.} \label{tab:hundred_models}
\end{wraptable}

We visualize the resulting predictive log-likelihoods in \cref{sf:forecasting_NLL}. Again, we observe clear improvements in the predictive performance of OBS over O-BMA, with similar results obtained by Soft-Bayes, EG, and ONS. As this is real-world financial data, we expect some amount of non-stationarity: indeed, we find that DMA improves upon BMA, but remains inferior to the OBS methods, and that D-ONS performs the best. We show that BMA collapses once more in this example in \cref{app:more_garch}, which can be seen in the evolution and final weights.

We additionally performed our online forecasting experiments with $25$ models from each class, for a total of $100$ models. Repeating across 3 random seeds, we obtain the results in \cref{tab:hundred_models}. We used an EG learning rate 10 times larger than that used in previous experiments, guided by the logarithmic increase in the optimal theoretical learning rate as the number of models increases. The results show that ONS and EG still outperform BMA, while Soft-Bayes performs worse. We posit that the issues with Soft-Bayes are due to the relatively short time horizon (on the order of \num{1000} time steps) relative to the number of models -- we are fundamentally trading off some fast adaptation for robustness when obtaining the gradient bound-free regret bounds of Soft-Bayes.

\begin{figure}[t]
     \centering
    \begin{subfigure}[b]{0.3\textwidth}
         \centering
         \includegraphics[width=\textwidth]{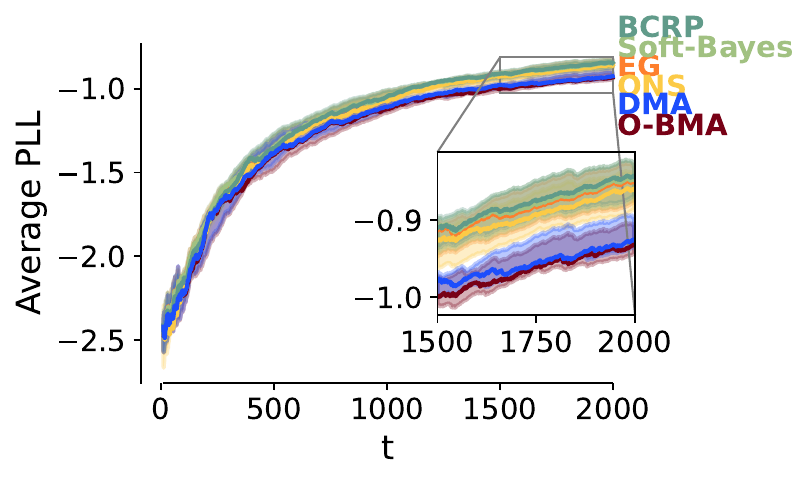}
         \caption{Bayesian Neural Network (\cref{sec:exp_bnn})}
         \label{sf:BONG_NLL}
     \end{subfigure}
     \begin{subfigure}[b]{0.3\textwidth}
         \centering
         \includegraphics[width=\textwidth]{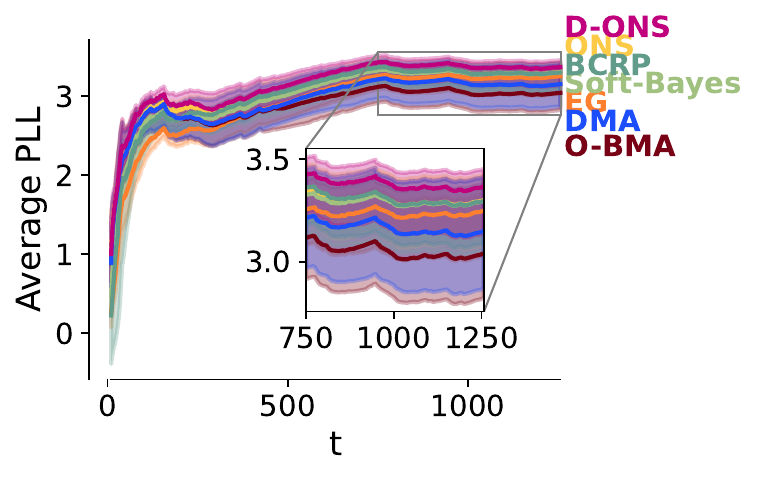}
         \caption{Forecasting with GARCH (\cref{sec:exp_forecasting})}
         \label{sf:forecasting_NLL}
     \end{subfigure}
        \caption{The average predictive log-likelihood (higher is better) in the MNIST and forecasting experiments, respectively. The method descriptions follow those in \cref{fig:exp_results}, with the addition of Soft-Bayes. Lines denote the median and shaded area represents  the 10th to 90th percentiles over 10 trials. The first 100 samples are suppressed for readability.}
        \label{fig:real_data}
\end{figure}

\subsection{Online Regression in Non-Stationary Environments} \label{sec:online_nonstationary}
A common issue in the online setting is non-stationarity, possibly due to covariate shift or concept drift. To illustrate the effectiveness of OBS in this setting, we apply OBS to the dynamic online ensemble of basis expansions (DOEBE) \citep{waxman2024dynamic}, which uses O-BMA to ensemble several online Gaussian process-based models. The DOEBE employs a linear basis approximation to GPs with random Fourier features \citep{lazaro2010sparse,rahimi2007random} and models non-stationary processes by imposing a random walk on these linear parameters, using variance $\sigma_{\text{rw}}^2$. This parameter is set to a default value in \citet{waxman2024dynamic}, but it is found to be quite important for performance on several real-world datasets.

We apply OBS instead of O-BMA to ensemble RFF-GPs with $\sigma_{\text{rw}} = 10^{-k}$, for $k \in \{ 0, 1, 2, 3, 4\}$. We use the same datasets as \citet{waxman2024dynamic} (excluding purely synthetic ones); a summary of datasets and experimental details is in \cref{app:more_doebe}. Of note are SARCOS and Kuka \#1, both of which are robotics datasets with covariate shift. For SARCOS and Kuka \#1 only, we use the smoothed version of EG \citep{helmbold1998line} with hyperparameters $\eta = 10^{-3}$ and $\delta = 10^{-2}$; otherwise, EG exhibits severe instabilities. 

We include numerical results in \cref{tab:nonstationary_results} and figures in \cref{app:more_doebe}. We again find OBS to be beneficial. Though EG might be sensitive to outliers and occasionally performs poorly (particularly on SARCOS and Kuka \#1),  ONS and Soft-Bayes are rather robust and consistently outperform O-BMA. DMA again outperforms O-BMA, and D-ONS performs best on almost every dataset. Interesting future work includes analyzing adaptive regret in non-stationary settings.

\begin{table}[t]
\centering
\caption{The average predictive log-likelihood (median with 10th and 90th percentiles) in the non-stationary experiments; \text{bolded} values are best, and \underline{underlined} values are second-best.}
\begin{tabular}{llccccc}
\toprule
Method Type & Method & Elevators & SARCOS & Kuka \#1 & CaData & CPU Small \\
\midrule
\multirow{2}{8em}{Online Baselines}
 & O-BMA & $-0.59_{\scriptstyle-0.60}^{\scriptstyle-0.58}$ & $0.77_{\scriptstyle0.74}^{\scriptstyle0.78}$ & $0.61_{\scriptstyle0.52}^{\scriptstyle0.62}$ & $-0.84_{\scriptstyle-0.85}^{\scriptstyle-0.83}$ & $0.33_{\scriptstyle0.32}^{\scriptstyle0.34}$ \\
 & DMA & $-0.59_{\scriptstyle-0.59}^{\scriptstyle-0.58}$ & $0.79_{\scriptstyle0.77}^{\scriptstyle0.80}$ & $0.69_{\scriptstyle0.63}^{\scriptstyle0.70}$ & $-0.81_{\scriptstyle-0.81}^{\scriptstyle-0.80}$ & $0.33_{\scriptstyle0.33}^{\scriptstyle0.34}$ \\
\midrule
 \multirow{1}{*}{Offline}
 & BCRP & $\textbf{-0.57}_{\scriptstyle-0.58}^{\scriptstyle-0.57}$ & $0.80_{\scriptstyle0.77}^{\scriptstyle0.81}$ & $0.66_{\scriptstyle0.62}^{\scriptstyle0.68}$ & $-0.81_{\scriptstyle-0.81}^{\scriptstyle-0.81}$ & $\textbf{0.35}_{\scriptstyle0.35}^{\scriptstyle0.37}$ \\
\midrule
\multirow{4}{6em}{Online Bayesian Stacking}
 & EG & $\underline{-0.58}_{\scriptstyle-0.58}^{\scriptstyle-0.57}$ & $0.78_{\scriptstyle0.59}^{\scriptstyle0.79}$ & $0.56_{\scriptstyle0.45}^{\scriptstyle0.67}$ & $-0.81_{\scriptstyle-0.81}^{\scriptstyle-0.81}$ & $0.35_{\scriptstyle0.30}^{\scriptstyle0.36}$ \\ 
 & Soft-Bayes & $-0.58_{\scriptstyle-0.58}^{\scriptstyle-0.58}$ & $0.80_{\scriptstyle0.77}^{\scriptstyle0.81}$ & $\underline{0.70}_{\scriptstyle0.67}^{\scriptstyle0.72}$ & $-0.81_{\scriptstyle-0.81}^{\scriptstyle-0.81}$ & $0.34_{\scriptstyle0.34}^{\scriptstyle0.36}$ \\ 
 & ONS & $-0.58_{\scriptstyle-0.59}^{\scriptstyle-0.58}$ & $\underline{0.80}_{\scriptstyle0.78}^{\scriptstyle0.82}$ & $0.70_{\scriptstyle0.66}^{\scriptstyle0.72}$ & $\underline{-0.80}_{\scriptstyle-0.81}^{\scriptstyle-0.80}$ & $0.34_{\scriptstyle0.34}^{\scriptstyle0.35}$ \\ 
 & D-ONS & $-0.58_{\scriptstyle-0.58}^{\scriptstyle-0.58}$ & $\textbf{0.81}_{\scriptstyle0.79}^{\scriptstyle0.82}$ & $\textbf{0.73}_{\scriptstyle0.69}^{\scriptstyle0.74}$ & $\textbf{-0.80}_{\scriptstyle-0.80}^{\scriptstyle-0.80}$ & $\underline{0.35}_{\scriptstyle0.35}^{\scriptstyle0.36}$ \\
\bottomrule
\end{tabular}
\label{tab:nonstationary_results}
\end{table}

\subsection{Additional Experiments}

We include several additional experiments in the appendices. Namely, in \cref{app:learning_rate}, we perform a sensitivity analysis on the learning rate in OCO algorithms. We come to the general conclusion that the learning rate does not significantly affect the results of our experiments for a set of \emph{a priori} sensible values.

Similarly, in \cref{tab:dma_with_other_lr}, we include ablation on the forgetting factor hyperparameter in DMA. We find that more aggressive forgetting can help in highly non-stationary environments, but that OBS methods (and, in particular, D-ONS) remain competitive, even without further hyperparameter tuning. 

\section{Discussion \& Conclusions} \label{sec:discussion}

In this work, we critically study the idea of ``Bayesian ensembles,'' especially in the online setting. 
We focused on two strategies for determining the weights in a linear ensemble: Bayesian model averaging (BMA) and Bayesian stacking (BS). In particular, we introduced and discussed OBS through the prequential principle. Through careful observation, we established a connection between the OBS problem and the OPS problem, allowing us to leverage the rich literature on online convex optimization. Below, we provide some practical guidance and limitations (\cref{sec:practical_guidance}), outlook (\cref{sec:outlook}), and brief conclusions in \cref{sec:conclusions}).

\subsection{Practical Guidance} \label{sec:practical_guidance}

\paragraph{When to Prefer OBS.} Our theory and experiments suggest that OBS is preferable in M-open scenarios, i.e., whenever the true data-generating model is not available, which we believe is typically the case in machine learning. When the problem is known to be M-closed, we reiterate previously-known theoretical results that O-BMA is optimal, though OBS, using many different online optimization algorithms, still provides competitive performance. Under anticipated and extreme non-stationarity, DMA with small forgetting factors is usually a viable approach, but otherwise, OBS using OCO algorithms designed for dynamic environments should be preferred.

\paragraph{Choosing Your OCO Algorithm.}
The theoretical and practical properties of OBS depend on the OCO algorithm used. Based on our experiments and analysis of the relevant theory, we provide the following recommendations:
\textbf{(1)} In scenarios with anticipated non-stationarities, D-ONS should be preferred.
\textbf{(2)} In more general scenarios where the anticipated non-stationarity is modest and gradual, guidance is somewhat more subjective. Based on our empirical experiments, we recommend ONS if the computation is not prohibitive, EG if the computation is extremely prohibitive and extreme non-stationarity is not a concern, and Soft-Bayes otherwise; however, we encourage practitioners to experiment in their particular domains.

\subsection{Outlook} \label{sec:outlook}
We view OBS as providing a flexible framework, allowing for exciting research along two different axes. 

First, the OBS connection provides an impactful application for new techniques developed in OCO. The connection also motivates a set of potentially interesting assumptions to incorporate into future regret bound analyses. For example, some recent work in OCO has focused on so-called ``data-dependent bounds,'' which can provide more optimistic regret bounds depending on the statistics of the incoming data \citep{tsai2023data,putta2025data}; understanding the statistics of predictive log-likelihoods in certain online Bayesian learning settings could thus be a natural place to develop new bounds. Online learning also serves as a natural place to apply research on switching regret \citep{pasteris2024online}.

Second, OBS is quite a general framework for ensembling Bayesian models. Though we considered a wide range of models and applications in this paper (including variational Bayesian neural networks, GPs via basis expansion approximations, and stochastic volatility models via sequential Monte Carlo), we anticipate many more. For example, our sequential Monte Carlo experiment shows the benefits of OBS over the standard ``bank of filters'' approach to ensembling particle filters, suggesting novel applications in filtering theory.

It is also interesting to consider connections to other probabilistic machine learning settings, including non-Bayesian online deep learning methods \citep{valkanas2025modl}, and applications of distributed or decentralized OCO \citep{yan2012distributed,mateos2014distributed} to decentralized Bayesian inference methods that rely on O-BMA \citep{llorente2025decentralized}.

\subsection{Conclusions} \label{sec:conclusions}
Furthermore, we illustrate how O-BMA optimizes a different loss than OBS leading to a novel empirical Bayesian analysis, and show how O-BMA-based regret bounds can often be adapted to OBS. Finally, we empirically validate our theoretical claims with both synthetic and real datasets,
showing that BMA collapses and that the proposed OBS algorithms deliver better performance. 

The connection established in this work between OBS and the well-studied problem of portfolio selection bridges optimization and statistics and suggests many interesting future directions. For example, in online or continual learning, one is often concerned with properties under regime changes; a corresponding future direction is to evaluate OBS with algorithms aimed at minimizing \emph{dynamic} or \emph{adaptive regret} \citep[Chapter 10]{hazan2022introduction}. Future statistical investigations include how to initialize ensembles well and whether the same rules of thumb apply as in BMA.

\section*{Acknowledgments} {This work was supported by the National Science Foundation under Award 2212506.}

\bibliography{main}

\begin{thebibliography}{64}
\providecommand{\natexlab}[1]{#1}
\providecommand{\url}[1]{\texttt{#1}}
\expandafter\ifx\csname urlstyle\endcsname\relax
  \providecommand{\doi}[1]{doi: #1}\else
  \providecommand{\doi}{doi: \begingroup \urlstyle{rm}\Url}\fi

\bibitem[Aastveit et~al.(2014)Aastveit, Gerdrup, Jore, and Thorsrud]{aastveit2014nowcasting}
Knut~Are Aastveit, Karsten~R Gerdrup, Anne~Sofie Jore, and Leif~Anders Thorsrud.
\newblock Nowcasting {GDP} in real time: A density combination approach.
\newblock \emph{Journal of Business \& Economic Statistics}, 32\penalty0 (1):\penalty0 48--68, 2014.

\bibitem[Agarwal et~al.(2006)Agarwal, Hazan, Kale, and Schapire]{agarwal2006algorithms}
Amit Agarwal, Elad Hazan, Satyen Kale, and Robert~E Schapire.
\newblock Algorithms for portfolio management based on the {N}ewton method.
\newblock In \emph{Proceedings of the 23rd international Conference on Machine Learning}, pp.\  9--16, 2006.

\bibitem[Ba{\c{s}}t{\"u}rk et~al.(2019)Ba{\c{s}}t{\"u}rk, Borowska, Grassi, Hoogerheide, and van Dijk]{bacsturk2019forecast}
Nalan Ba{\c{s}}t{\"u}rk, Agnieszka Borowska, Stefano Grassi, Lennart Hoogerheide, and Herman~K van Dijk.
\newblock Forecast density combinations of dynamic models and data driven portfolio strategies.
\newblock \emph{Journal of Econometrics}, 210\penalty0 (1):\penalty0 170--186, 2019.

\bibitem[Bates \& Granger(1969)Bates and Granger]{bates1969combination}
John~M Bates and Clive~WJ Granger.
\newblock The combination of forecasts.
\newblock \emph{Journal of the Operational Research Society}, 20\penalty0 (4):\penalty0 451--468, 1969.

\bibitem[Bradbury et~al.(2018)Bradbury, Frostig, Hawkins, Johnson, Leary, Maclaurin, Necula, Paszke, Vander{P}las, Wanderman-{M}ilne, and Zhang]{jax2018github}
James Bradbury, Roy Frostig, Peter Hawkins, Matthew~James Johnson, Chris Leary, Dougal Maclaurin, George Necula, Adam Paszke, Jake Vander{P}las, Skye Wanderman-{M}ilne, and Qiao Zhang.
\newblock {JAX}: composable transformations of {P}ython+{N}um{P}y programs, 2018.
\newblock URL \url{http://github.com/google/jax}.

\bibitem[Breiman(1996)]{breiman1996stacked}
Leo Breiman.
\newblock Stacked regressions.
\newblock \emph{Machine Learning}, 24:\penalty0 49--64, 1996.

\bibitem[Clyde \& Iversen(2013)Clyde and Iversen]{clyde2013bayesian}
Merlise Clyde and Edwin~S Iversen.
\newblock {Bayesian model averaging in the M-open framework}.
\newblock In \emph{{Bayesian Theory and Applications}}. Oxford University Press, 01 2013.
\newblock ISBN 9780199695607.
\newblock \doi{10.1093/acprof:oso/9780199695607.003.0024}.

\bibitem[Dawid(1984)]{dawid1984present}
A~Philip Dawid.
\newblock Present position and potential developments: Some personal views statistical theory the prequential approach.
\newblock \emph{Journal of the Royal Statistical Society: Series A (General)}, 147\penalty0 (2):\penalty0 278--290, 1984.

\bibitem[DeGroot(1974)]{degroot1974reaching}
Morris~H DeGroot.
\newblock Reaching a consensus.
\newblock \emph{Journal of the American Statistical association}, 69\penalty0 (345):\penalty0 118--121, 1974.

\bibitem[{Delve Developers}(1996)]{delve}
{Delve Developers}.
\newblock Delve datasets.
\newblock \url{https://www.cs.toronto.edu/~delve/data/datasets.html}, 1996.

\bibitem[Dietterich(2000)]{dietterich2000ensemble}
Thomas~G Dietterich.
\newblock Ensemble methods in machine learning.
\newblock In \emph{International Workshop on Multiple Classifier Systems}, pp.\  1--15. Springer, 2000.

\bibitem[Ding et~al.(2021)Ding, Yuan, and Jovanovi{\'c}]{ding2021discounted}
Dongsheng Ding, Jianjun Yuan, and Mihailo~R Jovanovi{\'c}.
\newblock Discounted online newton method for time-varying time series prediction.
\newblock In \emph{2021 American Control Conference (ACC)}, pp.\  1547--1552. IEEE, 2021.

\bibitem[Doucet et~al.(2001)Doucet, De~Freitas, and Gordon]{doucet2001introduction}
A.~Doucet, N.~De~Freitas, and N.~Gordon.
\newblock An introduction to sequential {M}onte {C}arlo methods.
\newblock \emph{Sequential Monte Carlo Methods in Practice}, pp.\  3--14, 2001.

\bibitem[Engle(2001)]{engle2001garch}
R.~Engle.
\newblock {GARCH} 101: The use of {ARCH}/{GARCH} models in applied econometrics.
\newblock \emph{Journal of Economic Perspectives}, 15\penalty0 (4):\penalty0 157--168, 2001.

\bibitem[Ferreira(2018)]{ferreira2018discussion}
Marco A.~R. Ferreira.
\newblock Contributed discussion to ``using stacking to average {B}ayesian predictive distributions''.
\newblock \emph{Bayesian Analysis}, 13\penalty0 (3):\penalty0 986--987, 2018.

\bibitem[Freund \& Schapire(1997)Freund and Schapire]{freund1997decision}
Yoav Freund and Robert~E Schapire.
\newblock A decision-theoretic generalization of on-line learning and an application to boosting.
\newblock \emph{Journal of Computer and System Sciences}, 55\penalty0 (1):\penalty0 119--139, 1997.

\bibitem[Genest \& Zidek(1986)Genest and Zidek]{genest1986combining}
Christian Genest and James~V Zidek.
\newblock Combining probability distributions: {A} critique and an annotated bibliography.
\newblock \emph{Statistical Science}, 1\penalty0 (1):\penalty0 114--135, 1986.

\bibitem[Geweke \& Amisano(2011)Geweke and Amisano]{geweke2011optimal}
John Geweke and Gianni Amisano.
\newblock Optimal prediction pools.
\newblock \emph{Journal of Econometrics}, 164\penalty0 (1):\penalty0 130--141, 2011.
\newblock ISSN 0304-4076.
\newblock \doi{https://doi.org/10.1016/j.jeconom.2011.02.017}.
\newblock Annals Issue on Forecasting.

\bibitem[Geweke \& Amisano(2012)Geweke and Amisano]{geweke2012prediction}
John Geweke and Gianni Amisano.
\newblock Prediction with misspecified models.
\newblock \emph{American Economic Review}, 102\penalty0 (3):\penalty0 482--486, 2012.

\bibitem[G{\'o}mez-Rubio et~al.(2020)G{\'o}mez-Rubio, Bivand, and Rue]{gomez2020bayesian}
Virgilio G{\'o}mez-Rubio, Roger~S Bivand, and H{\aa}vard Rue.
\newblock Bayesian model averaging with the integrated nested {L}aplace approximation.
\newblock \emph{Econometrics}, 8\penalty0 (2):\penalty0 23, 2020.

\bibitem[Hall \& Mitchell(2007)Hall and Mitchell]{hall2007combining}
Stephen~G Hall and James Mitchell.
\newblock Combining density forecasts.
\newblock \emph{International Journal of Forecasting}, 23\penalty0 (1):\penalty0 1--13, 2007.

\bibitem[Hazan(2022)]{hazan2022introduction}
Elad Hazan.
\newblock \emph{Introduction to Online Convex Optimization}.
\newblock MIT Press, 2022.

\bibitem[Hazan et~al.(2007)Hazan, Agarwal, and Kale]{hazan2007logarithmic}
Elad Hazan, Amit Agarwal, and Satyen Kale.
\newblock Logarithmic regret algorithms for online convex optimization.
\newblock \emph{Machine Learning}, 69\penalty0 (2):\penalty0 169--192, 2007.

\bibitem[Helmbold et~al.(1998)Helmbold, Schapire, Singer, and Warmuth]{helmbold1998line}
David~P Helmbold, Robert~E Schapire, Yoram Singer, and Manfred~K Warmuth.
\newblock On-line portfolio selection using multiplicative updates.
\newblock \emph{Mathematical Finance}, 8\penalty0 (4):\penalty0 325--347, 1998.

\bibitem[Hoeting et~al.(1999)Hoeting, Madigan, Raftery, and Volinsky]{hoeting1999bayesian}
Jennifer~A Hoeting, David Madigan, Adrian~E Raftery, and Chris~T Volinsky.
\newblock Bayesian model averaging: {A} tutorial (with comments by {M.} {C}lyde, {D}avid {D}raper and {E. I.} {G}eorge, and a rejoinder by the authors).
\newblock \emph{Statistical Science}, 14\penalty0 (4):\penalty0 382--417, 1999.

\bibitem[Jones et~al.(2024)Jones, Chang, and Murphy]{jones2024bayesian}
Matt Jones, Peter~G. Chang, and Kevin~Patrick Murphy.
\newblock Bayesian online natural gradient ({BONG}).
\newblock In \emph{The Thirty-eighth Annual Conference on Neural Information Processing Systems}, 2024.
\newblock URL \url{https://openreview.net/forum?id=E7en5DyO2G}.

\bibitem[Jore et~al.(2010)Jore, Mitchell, and Vahey]{jore2010combining}
Anne~Sofie Jore, James Mitchell, and Shaun~P Vahey.
\newblock Combining forecast densities from {VAR}s with uncertain instabilities.
\newblock \emph{Journal of Applied Econometrics}, 25\penalty0 (4):\penalty0 621--634, 2010.

\bibitem[Kakade \& Ng(2004)Kakade and Ng]{kakade2004online}
Sham~M Kakade and Andrew Ng.
\newblock Online bounds for {B}ayesian algorithms.
\newblock \emph{Advances in Neural Information Processing Systems}, 17, 2004.

\bibitem[Khaleghi et~al.(2013)Khaleghi, Khamis, Karray, and Razavi]{khaleghi2013multisensor}
Bahador Khaleghi, Alaa Khamis, Fakhreddine~O Karray, and Saiedeh~N Razavi.
\newblock Multisensor data fusion: {A} review of the state-of-the-art.
\newblock \emph{Information Fusion}, 14\penalty0 (1):\penalty0 28--44, 2013.

\bibitem[Koliander et~al.(2022)Koliander, El-Laham, Djuri{\'c}, and Hlawatsch]{koliander2022fusion}
G{\"u}nther Koliander, Yousef El-Laham, Petar~M Djuri{\'c}, and Franz Hlawatsch.
\newblock Fusion of probability density functions.
\newblock \emph{Proceedings of the IEEE}, 110\penalty0 (4):\penalty0 404--453, 2022.

\bibitem[L{\'a}zaro-Gredilla et~al.(2010)L{\'a}zaro-Gredilla, Quinonero-Candela, Rasmussen, and Figueiras-Vidal]{lazaro2010sparse}
Miguel L{\'a}zaro-Gredilla, Joaquin Quinonero-Candela, Carl~Edward Rasmussen, and An{\'\i}bal~R Figueiras-Vidal.
\newblock Sparse spectrum {G}aussian process regression.
\newblock \emph{The Journal of Machine Learning Research}, 11:\penalty0 1865--1881, 2010.

\bibitem[Le \& Clarke(2017)Le and Clarke]{le2017bayes}
Tri Le and Bertrand Clarke.
\newblock A {B}ayes interpretation of stacking for {M}-complete and {M}-open settings.
\newblock \emph{Bayesian Analysis}, 12\penalty0 (3):\penalty0 807--829, 2017.

\bibitem[LeCun et~al.(2010)LeCun, Cortes, and Burges]{lecun2010mnist}
Yann LeCun, Corinna Cortes, and CJ~Burges.
\newblock {MNIST} handwritten digit database.
\newblock \emph{ATT Labs [Online]. Available: http://yann.lecun.com/exdb/mnist}, 2, 2010.

\bibitem[Li \& Hoi(2014)Li and Hoi]{li2014online}
Bin Li and Steven~CH Hoi.
\newblock Online portfolio selection: {A} survey.
\newblock \emph{ACM Computing Surveys (CSUR)}, 46\penalty0 (3):\penalty0 1--36, 2014.

\bibitem[Llorente et~al.(2025)Llorente, Waxman, and Djuri{\'c}]{llorente2025decentralized}
Fernando Llorente, Daniel Waxman, and Petar~M Djuri{\'c}.
\newblock Decentralized online ensembles of gaussian processes for multi-agent systems.
\newblock In \emph{ICASSP 2025-2025 IEEE International Conference on Acoustics, Speech and Signal Processing (ICASSP)}, pp.\  1--5. IEEE, 2025.

\bibitem[Lu et~al.(2022)Lu, Karanikolas, and Giannakis]{lu2022incremental}
Qin Lu, Georgios~V Karanikolas, and Georgios~B Giannakis.
\newblock Incremental ensemble {G}aussian processes.
\newblock \emph{IEEE Transactions on Pattern Analysis and Machine Intelligence}, 45\penalty0 (2):\penalty0 1876--1893, 2022.

\bibitem[Mateos-N{\'u}nez \& Cort{\'e}s(2014)Mateos-N{\'u}nez and Cort{\'e}s]{mateos2014distributed}
David Mateos-N{\'u}nez and Jorge Cort{\'e}s.
\newblock Distributed online convex optimization over jointly connected digraphs.
\newblock \emph{IEEE Transactions on Network Science and Engineering}, 1\penalty0 (1):\penalty0 23--37, 2014.

\bibitem[Meier et~al.(2014)Meier, Hennig, and Schaal]{meier2014incremental}
Franziska Meier, Philipp Hennig, and Stefan Schaal.
\newblock Incremental local {G}aussian regression.
\newblock \emph{Advances in Neural Information Processing Systems}, 27, 2014.

\bibitem[Minka(2000)]{minka2000bayesian}
Thomas~P Minka.
\newblock Bayesian model averaging is not model combination.
\newblock Technical report, MIT Media Lab, 2000.
\newblock URL \url{https://tminka.github.io/papers/minka-bma-isnt-mc.pdf}.

\bibitem[Nemirovsky et~al.(1983)Nemirovsky, Yudin, and Dawson]{nemirovsky1983wiley}
AS~Nemirovsky, DB~Yudin, and ER~Dawson.
\newblock \emph{Problem Complexity and Method Efficiency in Optimization}.
\newblock John Wiley \& Sons, Inc New York, 1983.

\bibitem[{Objax Developers}(2020)]{objax2020github}
{Objax Developers}.
\newblock {Objax}, 2020.
\newblock URL \url{https://github.com/google/objax}.

\bibitem[Orabona(2019)]{orabona2019modern}
Francesco Orabona.
\newblock A modern introduction to online learning.
\newblock \emph{arXiv preprint arXiv:1912.13213}, 2019.

\bibitem[Orseau et~al.(2017)Orseau, Lattimore, and Legg]{orseau2017soft}
Laurent Orseau, Tor Lattimore, and Shane Legg.
\newblock Soft-{B}ayes: Prod for mixtures of experts with log-loss.
\newblock In \emph{International Conference on Algorithmic Learning Theory}, pp.\  372--399. PMLR, 2017.

\bibitem[Pace \& Barry(1997)Pace and Barry]{pace1997sparse}
R~Kelley Pace and Ronald Barry.
\newblock Sparse spatial autoregressions.
\newblock \emph{Statistics \& Probability Letters}, 33\penalty0 (3):\penalty0 291--297, 1997.

\bibitem[Pasteris et~al.(2024)Pasteris, Hicks, Mavroudis, and Herbster]{pasteris2024online}
Stephen Pasteris, Chris Hicks, Vasilios Mavroudis, and Mark Herbster.
\newblock Online convex optimisation: The optimal switching regret for all segmentations simultaneously.
\newblock \emph{Advances in Neural Information Processing Systems}, 37:\penalty0 78278--78298, 2024.

\bibitem[Putta \& Agrawal(2025)Putta and Agrawal]{putta2025data}
Sudeep~Raja Putta and Shipra Agrawal.
\newblock Data dependent regret bounds for online portfolio selection with predicted returns.
\newblock In \emph{36th International Conference on Algorithmic Learning Theory}, 2025.

\bibitem[Raftery et~al.(2010)Raftery, K{\'a}rn{\`y}, and Ettler]{raftery2010online}
Adrian~E Raftery, Miroslav K{\'a}rn{\`y}, and Pavel Ettler.
\newblock Online prediction under model uncertainty via dynamic model averaging: Application to a cold rolling mill.
\newblock \emph{Technometrics}, 52\penalty0 (1):\penalty0 52--66, 2010.

\bibitem[Rahimi \& Recht(2007)Rahimi and Recht]{rahimi2007random}
Ali Rahimi and Benjamin Recht.
\newblock Random features for large-scale kernel machines.
\newblock \emph{Advances in Neural Information Processing Systems}, 20, 2007.

\bibitem[Rasmussen \& Williams(2005)Rasmussen and Williams]{rasmussen_and_williams}
Carl~Edward Rasmussen and Christopher K.~I. Williams.
\newblock \emph{{G}aussian Processes for Machine Learning}.
\newblock The MIT Press, 11 2005.
\newblock ISBN 9780262256834.
\newblock \doi{10.7551/mitpress/3206.001.0001}.

\bibitem[Robbins(1964)]{robbins1964empirical}
Herbert Robbins.
\newblock The empirical bayes approach to statistical decision problems.
\newblock \emph{The Annals of Mathematical Statistics}, 35\penalty0 (1):\penalty0 1--20, 1964.

\bibitem[Stone(1961)]{stone1961opinion}
Mervyn Stone.
\newblock The opinion pool.
\newblock \emph{The Annals of Mathematical Statistics}, pp.\  1339--1342, 1961.

\bibitem[Torgo(2024)]{elevators}
Lu\'s Torgo.
\newblock Regression datasets, 2024.
\newblock URL \url{https://www.dcc.fc.up.pt/~ltorgo/Regression/DataSets.html}.
\newblock Accessed 2024-08-15.

\bibitem[Tsai et~al.(2023)Tsai, Lin, and Li]{tsai2023data}
Chung-En Tsai, Ying-Ting Lin, and Yen-Huan Li.
\newblock Data-dependent bounds for online portfolio selection without {L}ipschitzness and smoothness.
\newblock \emph{Advances in Neural Information Processing Systems}, 36:\penalty0 62764--62791, 2023.

\bibitem[Valkanas et~al.(2025)Valkanas, Oreshkin, and Coates]{valkanas2025modl}
Antonios Valkanas, Boris~N. Oreshkin, and Mark Coates.
\newblock {MODL}: Multilearner online deep learning.
\newblock In \emph{The 28th International Conference on Artificial Intelligence and Statistics}, 2025.
\newblock URL \url{https://openreview.net/forum?id=WspiEX6v3r}.

\bibitem[Van~Erven et~al.(2020)Van~Erven, Van~der Hoeven, Kot{\l}owski, and Koolen]{van2020open}
Tim Van~Erven, Dirk Van~der Hoeven, Wojciech Kot{\l}owski, and Wouter~M Koolen.
\newblock Open problem: Fast and optimal online portfolio selection.
\newblock In \emph{Conference on Learning Theory}, pp.\  3864--3869. PMLR, 2020.

\bibitem[Van~Vaerenbergh et~al.(2012)Van~Vaerenbergh, L{\'a}zaro-Gredilla, and Santamar{\'\i}a]{van2012kernel}
Steven Van~Vaerenbergh, Miguel L{\'a}zaro-Gredilla, and Ignacio Santamar{\'\i}a.
\newblock Kernel recursive least-squares tracker for time-varying regression.
\newblock \emph{IEEE Transactions on Neural Networks and Learning Systems}, 23\penalty0 (8):\penalty0 1313--1326, 2012.

\bibitem[Vehtari et~al.(2024)Vehtari, Simpson, Gelman, Yao, and Gabry]{vehtari2024psis}
Aki Vehtari, Daniel Simpson, Andrew Gelman, Yuling Yao, and Jonah Gabry.
\newblock Pareto smoothed importance sampling.
\newblock \emph{Journal of Machine Learning Research}, 25\penalty0 (72):\penalty0 1--58, 2024.
\newblock URL \url{http://jmlr.org/papers/v25/19-556.html}.

\bibitem[Vinkler(2024)]{universalportfolios_github}
Mojmir Vinkler.
\newblock Universal portfolios.
\newblock \url{https://github.com/Marigold/universal-portfolios}, 2024.
\newblock Accessed 2024-08-15.

\bibitem[Vovk(2001)]{vovk2001competitive}
Volodya Vovk.
\newblock Competitive on-line statistics.
\newblock \emph{International Statistical Review}, 69\penalty0 (2):\penalty0 213--248, 2001.

\bibitem[Waxman \& Djuri\'c(2024)Waxman and Djuri\'c]{waxman2024dynamic}
Daniel Waxman and Petar~M. Djuri\'c.
\newblock Dynamic online ensembles of basis expansions.
\newblock \emph{Transactions on Machine Learning Research}, 2024.
\newblock ISSN 2835-8856.
\newblock URL \url{https://openreview.net/forum?id=aVOzWH1Nc5}.

\bibitem[Yan et~al.(2012)Yan, Sundaram, Vishwanathan, and Qi]{yan2012distributed}
Feng Yan, Shreyas Sundaram, SVN Vishwanathan, and Yuan Qi.
\newblock Distributed autonomous online learning: Regrets and intrinsic privacy-preserving properties.
\newblock \emph{IEEE Transactions on Knowledge and Data Engineering}, 25\penalty0 (11):\penalty0 2483--2493, 2012.

\bibitem[Yao et~al.(2018)Yao, Vehtari, Simpson, and Gelman]{yao2018using}
Yuling Yao, Aki Vehtari, Daniel Simpson, and Andrew Gelman.
\newblock Using stacking to average {B}ayesian predictive distributions (with discussion).
\newblock \emph{Bayesian Analysis}, 13\penalty0 (3):\penalty0 917--1003, 2018.

\bibitem[Yuan \& Lamperski(2020)Yuan and Lamperski]{yuan2020trading}
Jianjun Yuan and Andrew Lamperski.
\newblock Trading-off static and dynamic regret in online least-squares and beyond.
\newblock In \emph{Proceedings of the AAAI Conference on Artificial Intelligence}, volume~34, pp.\  6712--6719, 2020.

\bibitem[Zellner(1988)]{zellner1988optimal}
Arnold Zellner.
\newblock Optimal information processing and {B}ayes's theorem.
\newblock \emph{The American Statistician}, 42\penalty0 (4):\penalty0 278--280, 1988.

\end{thebibliography}
\bibliographystyle{tmlr}

\appendix\clearpage

\section{More Details on Regret Analysis} \label{app:regret_analysis}
One of the main benefits of connecting OBS with OPS is the numerous regret bounds available in the literature. Stating regret bounds for applications of OBS in general is not possible, as the available regret bounds depend on both the OCO algorithm and the class of models being ensembled. Below, we first state some general bounds reflecting the state of the art in OPS. We then provide an example of how these bounds may be applied, showing a regret analysis for OBS applied to an ensemble of approximate Gaussian processes.

\subsection{Regret Bounds from Portfolio Selection}

Different choices of algorithm and different assumptions on the predictive densities $\V{p}_t$ affect the obtainable regret analysis (c.f. \citet[Table 1]{van2020open}). However, what is generally possible in \emph{streaming} algorithms (i.e., those whose runtime is linear in $T$) without making further assumptions on the loss is regret proportional to $\mathcal{O}(\sqrt{T}).$

By making further assumptions on the loss function --- and in particular, assuming that a market variability parameter $\alpha$ exists --- we can obtain even sharper bounds, often for simple algorithms. For example, let $G_{p}$ be a bound on the $p$-norm of the gradients (which follows from the existence of a bounded market variability parameter $\alpha$). Then EG guarantees regret proportional to $\sqrt{T}$ with an appropriate choice of learning rate \citep[Theorem 4.1]{helmbold1998line}. %
ONS provides even better logarithmic regret \citep[Theorem 1]{agarwal2006algorithms}.%

In the absence of a market variability parameter $\alpha$, both EG and ONS may be modified with a ``smoothing'' analysis, which artificially bounds the gradients, with the resulting EG regret bounds being proportional to $T^{3/4}$ \citep[Theorem 4.2]{helmbold1998line}, and the resulting ONS regret bounds being proportional to $\sqrt{T \log T}$ \citep[Theorem 1]{agarwal2006algorithms}. Soft-Bayes can provide regret bounds independent of $G$ and without prior knowledge using time-varying learning rates, reducing dependence on $T$ to $\sqrt{T}$.

How to apply these bounds to achieve results for learning algorithms is an interesting question and depends on the learning algorithm and the type of bound to prove. Below, we provide a simple example that shows OBS may preserve the asymptotic properties of existing regret bounds, even when regret is measured with respect to a single expert rather than a mixture.

\subsection{Regret of Online BMA} \label{app:bma_lemma}
We now prove \cref{eq:bma_lemma}.

\begin{lemma31*}
    Let the regret of the BMA mixture with respect to the best individual model be defined as
    \begin{equation} \label{eq:bma_regret}
        R_T = \sum_t \log p_{k^{*}}(y_{t} \given \V{x}_t, \mathcal{D}_{t-1}) -\sum_t \log\left(\sum_k w_{t, k} p_k(y_t \given \V{x}_t, \mathcal{D}_{t-1}) \right),
    \end{equation}
    where $k_*$ is the best model. Then $R_T$ is related to an evidence gain in $M_{k^{*}}$,
    \begin{equation}
        R_T = \log \frac{\Pr(\mathcal{M}_{k^*} \given \mathcal{D}_T)}{\Pr(\mathcal{M}_{k^*})}.
    \end{equation}
\end{lemma31*}
\begin{proof}
    This follows quickly from the definition of the O-BMA weights \cref{eq:BMA_weights_online}. Denoting $\ell^{\text{(BMA)}}_t = -\sum_t \log\left(\sum_k w_{t, k} p_k(y_t \given \V{x}_t, \mathcal{D}_{t-1} \right)$ and $l_t^{(k^{*})} = -\log p_{k^{*}}(y_{t} \given \V{x}_t, \mathcal{D}_{t-1})$, note that
        \begin{equation*}
        \frac{w^{(k)}_{t}}{w^{(k)}_{t-1}} = \frac{\exp(-l^{(k_*)}_t)}{\exp(-\ell_t)}.
    \end{equation*}
    Therefore, we arrive at a telescoping term,
    \begin{equation*}
        \exp\left(\sum_{t=1}^T \ell^\text{(BMA)}_t - l_{t}^{(k)}\right) = \frac{\cancel{w_1^{(k_*)}}}{w_0^{(k_*)}} \frac{\cancel{w_1^{(k_*)}}}{\cancel{w_2^{(k_*)}}} \cdots \frac{w_T^{(k_*)}}{\cancel{w_{T-1}^{(k_*)}}} = \frac{w^{(k_*)}_T}{w^{(k_*)}_0}.
    \end{equation*}
    Taking the logarithm and plugging in the closed form values \cref{eq:BMA_weights} proves the proposition.
\end{proof}

 \subsection{Example: Online Ensembles of Basis Expansions} \label{app:gp_regret_example}
A recent example of O-BMA being applied in machine learning is in the ensembling of (approximate) Gaussian processes (GPs). In particular, \citet{lu2022incremental} proposes ensembling approximate GPs with O-BMA, with \citet{waxman2024dynamic} generalizing the algorithm and bounds to more general linear basis expansions. They both prove regret bounds relative to any expert with fixed parameters, a setting in which BMA is the appropriate optimizer. 

It will first be useful to state the following lemma:
\begin{lemma}[\citet{kakade2004online}, Theorem 2.2] \label{lemma:kakade}
    Let $\ell(\cdot ; y_\tau)$ denote the negative log-likelihood, and assume it is $\mathcal{C}^2$ with second derivatives bounded by $c\in \mathbb{R}$. We are concerned with the negative predictive log-likelihood $\ell_t$ of a Bayesian linear model with basis expansion $\phi(\cdot) \colon \mathbb{R}^d \to \mathbb{R}^F$, using prior $p(\V\theta) = \mathcal{N}(\V\theta; \V{0}, \sigma_{\V\theta}^2, \V{I}_{F})$.
    
    Let $\V{x}_1, \dots, \V{x}_T$ be a sequence of inputs such that $\lVert \phi(\V{x}_t) \rVert$ is bounded by $1$ for all $t$. Then we have the following bound between the cumulative log-loss of the Bayesian estimator and the log-loss for any fixed value $\V\theta_*$: 
    \begin{equation*}
        \sum_{t=1}^T \ell_t - \ell\left(\phi(\V{x}_t)^\top \V\theta_*; y_t\right) \leq \frac{\lVert \V\theta_* \rVert^2}{2 \sigma_{\V\theta}^2} + \frac{F}{2} \log\left(1 + \frac{T c \sigma_{\V\theta}^2}{F}\right).
    \end{equation*}
\end{lemma}

In terms of the time horizon, the bound promised by \cref{lemma:kakade} is clearly $\mathcal{O}(\log T)$ in $T$.

We now formally state and prove the theorem.

\begin{theorem} \label{thm:oebe_bound}
    Let the negative log-likelihood $\mathcal{\ell}(\cdot; y_\tau)$ be $\mathcal{C}^2$ with second derivatives bounded by $c \in \mathbb{R}$. We then consider an online ensemble of basis functions \citep{waxman2024dynamic} with basis expansions $\phi^{(k)} \colon \mathbb{R}^{d_X} \to \mathbb{R}^{F_k}$ and priors $p(\mathbf{\theta}^{(k)}) = \mathcal{N}(\V\theta^{(k)} ; \mathbf{0}, {\sigma_{\V\theta}^{(k)}}^2 \mathbf{I}_{F_k})$ for $k \in \{1, \dots, K\}$. Further, assume that $\lVert \phi^{(k)} \rVert$ is bounded by $1$.

    We will consider the log-loss of the ensemble at some pre-selected time $T$, denoted $\ell_T$, and its regret with respect to the performance of any single expert $k$ with a fixed parameter $\V\theta_*^{(k)}$. Then \textbf{(a)} using O-BMA, the resulting regret is $\mathcal{O}(\log T)$ in $T$; \textbf{(b)} if we further assume that the log-loss is bounded, then the resulting regret remains $\mathcal{O}(\log T)$ using OBS with ONS as the optimizer. 
\end{theorem}
\begin{proof}
    The result \textbf{(a)} is proved directly in \citet[Theorem 1]{waxman2024dynamic}, which is adapted from \citet[Lemma 2]{lu2022incremental}. What remains to show is \textbf{(b)}. To emphasize the similarity in proving the O-BMA and OBS results, we will present them side by side.
    
    Let $l_t^{(k)}$ denote the log-loss of the $k$th expert at time $t$. The proof then proceeds in two steps: first, bounding the loss of the ensemble estimate to any individual expert, and then applying \cref{lemma:kakade} and combining the bounds.
    
    \paragraph{Bounding Ensemble Losses to Experts} We proceed by first bounding the ensemble loss to the loss of any individual expert. Beginning with O-BMA, and proceeding identically to  \cref{eq:bma_lemma}, \citet{lu2022incremental} make the observation that
    \begin{equation*}
        \frac{w^{(k)}_{t-1}}{w^{(k)}_{t}} = \frac{\exp(-\ell_t)}{\exp(-l^{(k)}_t)}.
    \end{equation*}
    Therefore, using initial weights $w^{(k)} = 1 / K$, 
    \begin{equation*}
        \exp\left(-\sum_{t=1}^T \ell^\text{(BMA)}_t +l_{t}^{(k)}\right) = \frac{w_0^{(k)}}{\cancel{w_1^{(k)}}} \frac{\cancel{w_1^{(k)}}}{\cancel{w_2^{(k)}}} \cdots \frac{\cancel{w_{T-1}^{(k)}}}{w_T^{(k)}} = \frac{1}{Mw^{(k)}_T}.
    \end{equation*}
    Thus, we can bound the regret of O-BMA with any individual expert as
    \begin{equation*}
        \sum_{t=1}^T \ell^\text{(O-BMA)}_t - l_{t}^{(k)} \leq \log M.
    \end{equation*}

    For OBS, we apply the bound of \citet[Theorem 2]{hazan2007logarithmic}, yielding
    \begin{equation*}
        \sum_{t=1}^T \ell^\text{(OBS)}_t - l_{t}^{(k)} \leq \sum_{t=1}^T \ell^\text{(OBS)}_t - \left(\sum_m w^*_m l_{t}^{(k)}\right) \leq \mathcal{O}(\log T),
    \end{equation*}
    where Big O notation absorbs constants related to $M$ and the maximum gradient norm.

    \paragraph{Bounding Expert Loss} Now, we bound the cumulative loss $\sum_{t=1}^T l_{t}^{(k)}$ to the loss using the fixed parameter $\V\theta_*^{(k)}$. This is achieved by applying \cref{lemma:kakade} to the $k$th expert, which results in 
    \begin{equation*}
        \sum_{t=1}^T l_{t}^{(k)}  - \ell\left(\phi^{(k)}(\V{x}_t)^\top \V\theta_*; y_t\right) \leq \frac{\lVert \V\theta_* \rVert^2}{2 {\sigma_{\V\theta}^{(k)}}^2} + \frac{F_k}{2} \log\left(1 + \frac{T c{\sigma_{\V\theta}^{(k)}}^2}{F_k}\right) \in \mathcal{O}(\log T).
    \end{equation*}

Combining the two bounds completes the proof.
\end{proof}

\section{Experimental Setup} \label{sec:experimental_setup_and_code}
In this appendix, we discuss our experimental setup, including the OCO algorithms and hyperparameters used throughout.

\subsection{Experimental Setup}
All experiments were conducted using Ubuntu 22.04 with an Intel i9-9900K CPU with 128 GB of RAM and two NVIDIA Titan RTX GPUs. 
Code for our methods and experiments is available online under an MIT License.
\footnote{\url{https://github.com/DanWaxman/OnlineBayesianStacking}, MIT License}
Implementations were in Jax/Objax \citep{jax2018github,objax2020github}\footnote{\url{https://github.com/google/objax}, Apache 2.0 License} based on modifying the codes of \citet{waxman2024dynamic}\footnote{\url{https://github.com/DanWaxman/DynamicOnlineBasisExpansions}, MIT License} and \citet{jones2024bayesian}\footnote{\url{https://github.com/petergchang/bong/tree/main}, MIT License}. Optimization algorithm implementations were adapted from and tested against the Universal Portfolios library \citep{universalportfolios_github}\footnote{\url{https://github.com/Marigold/universal-portfolios}, MIT License}.

\subsection{Baselines}
\paragraph{Best Constantly Rebalanced Portfolio} The BCRP is the so-called ``static'' solution to our stacking problem, equivalent to the offline versions of Bayesian stacking presented in \citet{hall2007combining,geweke2011optimal}. The resulting weights are a solution to \cref{eq:bayesian_stacking_optimization_problem_timeseries}, i.e.,
\begin{equation}
    \V{w}^* \coloneq \argmax_{\V{w} \in \mathbb{S}^K} \sum_{t=1}^{T} \log \sum_{k=1}^K w_k p_k(y_{t} \given \V{x}_{t}, \mathcal{D}_{1:t-1}).
\end{equation}
In regression problems without concept drift, this forms a reasonable baseline against which to measure regret.

\paragraph{Online Bayesian Model Averaging} Our main baseline is O-BMA, which we consider to be the standard online ensembling approach in Bayesian applications. The weights are updated as \cref{eq:BMA_weights_online}, i.e.,
\begin{equation} 
    w_{t+1,k}^{\text{BMA}} = \frac{ w_{t, k} p_k(y_{t+1} \given \V{x}_{t+1}, \mathcal{D}_t)}{\sum_k w_{t, k} p_k(y_{t+1} \given \V{x}_{t+1}, \mathcal{D}_t)}.
\end{equation}

\paragraph{Dynamic Model Averaging} In potentially non-stationary environments, we also compare to dynamic model averaging (DMA) \citep{raftery2010online}, which is essentially O-BMA with a forgetting factor $\gamma\in(0, 1)$:
\begin{equation} \label{eq:DMA_weights_online}
    w_{t+1,k}^{\text{DMA}} = \frac{ w_{t, k}^\gamma p_k(y_{t+1} \given \V{x}_{t+1}, \mathcal{D}_t)}{\sum_k w_{t, k}^\gamma p_k(y_{t+1} \given \V{x}_{t+1}, \mathcal{D}_t)}.
\end{equation}
While DMA loses some of the nice statistical properties of O-BMA, it can be more robust in non-stationary scenarios, where the best expert may change over time. For the forgetting factor, we follow \citet{raftery2010online} and use a value of $\gamma = 0.99$.

\subsection{Portfolio Selection Algorithms}
We include several OPS algorithms in our comparisons. We provide brief overviews of each and our default set of hyperparameters, below.

\paragraph{Exponentiated Gradients} The EG algorithm \citep{helmbold1998line} for portfolio selection follows the updates in \cref{eq:exponentiated_gradients_portfolio}, i.e.,
\begin{equation}
    \V{w}_{t+1} = \frac{\V{w}_t \odot \exp(\eta \V{r}_t / \V{w}_t^\top \V{r}_t)}{\lVert \V{w}_t \odot \exp(\eta \V{r}_t / \V{w}_t^\top \V{r}_t) \rVert_1},
\end{equation}
where $\eta$ is a learning rate parameter. We choose a default value of $\eta = 10^{-2}$ in our experiments.

\paragraph{Soft-Bayes} The Soft-Bayes algorithm \citep{orseau2017soft} provides several different potential updates, with regret guarantees independent of a market variability parameter $\alpha$. For our purposes, one useful formulation is the ``online'' variant of \citet[Section 6]{orseau2017soft}, which updates weights using a time-varying learning rate $\eta_t$ as
\begin{equation}
    w_{t+1, k} = w_{t, k} \left(1 - \eta_t + \eta_t \frac{r_t^k}{\sum_k w_{t, k} r_{t, k} }\right) \frac{\eta_{t+1}}{\eta_t} + \left(1 - \frac{\eta_{t+1}}{\eta_t}\right)w_{0, k}.
\end{equation}
Square-root regret is then guaranteed \citep[Theorem 10]{orseau2017soft} with the learning rate 
\begin{equation}
    \eta_t = \frac{\log K}{2 K t},
\end{equation}
which we use in our experiments.

\paragraph{Online Newton Step} In our experiments, we use the form of ONS specialized for the OPS problem in \citet{agarwal2006algorithms}. This algorithm requires parameters $\eta, \beta, \delta$ and keeps track of quantities $\V{A}_{t}$ and $\V{b}_t$, defined as 
\begin{align}
        \V{A}_t &= \sum_{\tau=1}^t -\nabla^2\left[\log (\V{w}_\tau \cdot \V{r}_t) \right] + \mathbb{I}_K;\\\V{b}_t &= \left(1 + \frac{1}{\beta}\right) \sum_{\tau=1}^t \nabla\left[\log (\V{w}_\tau \cdot \V{r}_t) \right].
\end{align}
Weights are then obtained at each iteration by projecting onto the simplex,
\begin{equation}
    \V{w}_{t+1} = \Pi_{\mathbb{S}^K}^{\V{A}_{t}}\left(\delta \V{A}_t^{-1} \V{b}_t\right),
\end{equation}
where the projection operator $\Pi_{\mathbb{S}^K}^{\V{A}_{t}}$ is defined as
\begin{equation}
    \Pi_{\mathbb{S}^K}^{\V{A}_{t}}(\V{v}) = \argmin_{\V w} (\V{w} - \V{v})^\top \V{A} (\V{w} - \V{v}).
\end{equation}
Additionally, to help with very small market variability parameters, gradients may be artificially ``smoothed'' as
\begin{equation}
    \Tilde{\V{w}}_t = (1 - \eta) \V{w}_t + \frac{1}{K} \mathbf{1}.
\end{equation}
We use default values of $\delta = 0.8$ and $\beta = \eta = 10^{-2}$.

\paragraph{Discounted Online Newton Step} The D-ONS algorithm \citep{yuan2020trading,ding2021discounted} is similar to the ONS algorithm (as in \citet{hazan2007logarithmic}), but it uses a forgetting factor for the second-order information. We use the formulation of \citet{ding2021discounted}, which the authors claim is more numerically stable than the discounting factor used in \citet{yuan2020trading}. The D-ONS algorithm depends on parameters $\eta > 0$ and $\gamma \in (0, 1)$ and keeps track of a quantity 
\begin{equation}
    P_{t+1} = (1- \gamma) P_0 + \gamma P_{t} - \nabla^2_t \left[\log (\V{w}_\tau \cdot \V{r}_t)\right].
\end{equation}
Updates are performed as 
\begin{equation}
    \V{w}_{t+1} = \Pi^{P_t}_{\mathbb{S}^K} \left(\V{w}_t - \frac{1}{\eta} P_t^{-1} \nabla\left[\log (\V{w}_\tau \cdot \V{r}_t)\right]\right).
\end{equation}
We use default values of $\eta = 1.0$ and $\gamma=0.99$.

\section{Details and Weights for the Subset of Linear Regressors Experiment} \label{app:more_linear_regressors}

In this appendix, we provide more details on the experimental setup and results for the subset of the linear regression experiment in \cref{sec:exp_subset}.

\subsection{Experimental Details}

\paragraph{Data Generation}
Following \citet{breiman1996stacked} and \citet[Section 4.2]{yao2018using}, we generate 6000 data points i.i.d. according to:
\begin{align*}
    \V{x}_t &\sim \mathcal{N}(5\cdot \V{1}_{15}, \V{I}_{15}), \\
    y_t \given \V{x}_t &\sim \mathcal{N}(\V{x}_t^\top \V{\theta}, 1).
\end{align*}
The ground truth parameters $\V{\theta} \in \mathbb{R}^{15}$ are generated such that all 15 components are individually weak predictors, and the signal-to-noise ratio is $0.8$; refer to \citet[Section 4.2]{yao2018using} for the precise procedure used to generate $\V{\theta}$.

\paragraph{Model Descriptions}
We consider two ensembles, each comprising $K=15$ Bayesian linear regression models:
\begin{itemize}
    \item \textbf{``Open'' setting:} Model $k$ ($k=1,\dots,15$) is a univariate regression attempting to learn the function $y = \theta_k x_k$, using only the $k$-th component of $\V{x}_t$.
    \item \textbf{``Closed'' setting:} Model $k$ ($k=1,\dots,15$) uses the first $k$ components of $\V{x}_t$, i.e., attempting to learn $y = \sum_{j=1}^k \theta_j x_j$. Model $k=15$ uses all features and corresponds to the true data-generating family.
\end{itemize}
Standard conjugate priors were used for the parameters $\theta_k$ (or vectors thereof) in each model.

\paragraph{Hyperparameter Specification}
For both scenarios, the first 1000 data points were used to set the hyperparameters of the base linear regression models (e.g., prior variances) via empirical Bayes (type-II maximum likelihood on the marginal likelihood over these 1000 points). The subsequent 5000 points were processed online. This was accomplished by modifying the code of \citet{waxman2024dynamic}.

The OCO hyperparameters for OBS were set to reasonable default values and were not tuned further:
\begin{itemize}
    \item Exponentiated Gradients (EG): Learning rate $\eta = 10^{-2}$.
    \item Online Newton Step (ONS): Parameters $\delta = 0.8$, $\eta = 10^{-2}$, $\beta = 10^{-2}$.
\end{itemize}

\section{Details and Weights for the Online Variational Inference Experiment} \label{app:more_mnist}

In this appendix, we provide more details on the experimental setup and results for the online variational inference experiment in \cref{sec:exp_bnn}. 

\subsection{Experimental Details}
As mentioned in the main text, we use the \texttt{bong-dlr10-ef\_lin} model from \citet{jones2024bayesian}, which had the best performance among the variants tested in their experiments. The BONG performs variational inference using natural gradients, updating variational parameters $\V{\psi}_t$ as
\begin{equation*}
    \V{\psi}_{t+1} = \V{\psi}_t + \V{F}_{\V{\psi}_t}^{-1} \nabla_{\V{\psi}_t} \mathbb{E}_{\V{\theta}_t \sim q_{\V{\psi}_{t}}}[ \log p(\V{y}_t \given \mathbf{x}_t, \V{\theta}_t)],
\end{equation*}
where $\V{F}_{\V{\psi}_t}$ is the Fisher information matrix and $q_{\V{\psi}_{t}}$ is the variational posterior. The \texttt{bong-dlr10-ef\_lin} variant uses a variational family of Gaussians with low-rank covariance matrices (diagonal + a rank $10$ matrix) and approximates the predictive likelihood through a linearization using the empirical Fisher information matrix. We use the authors' implementation of BONG.\footnote{\url{https://github.com/petergchang/bong/}, MIT License}

Our experimental setup generally mirrors that of \citet{jones2024bayesian}, with the following exceptions: we use a feedforward neural network with layers of width $64$ and $32$. The prior mean is sampled from the default \texttt{flax} initialization, and we form an ensemble over the prior variances $\sigma^2_0 \in \{10^{-2}, 10^{-1}, 10^{0}, 10^{1}\}$. Each ``trial'' corresponds to a different random shuffling of the training data, with the first 2000 data points used for inference.

\subsection{Weight Evolutions and Final Weights}
Once again, we visualize the resulting weights in \cref{fig:weights_bong} and \cref{fig:weights_bong_evolution}. We find a remarkable similarity between the final weights in all OCO algorithms tested and the BCRP solution, and that once again, BMA incorrectly collapses to a single model.

\begin{figure*}
    \centering
    \includegraphics[width=0.7\textwidth]{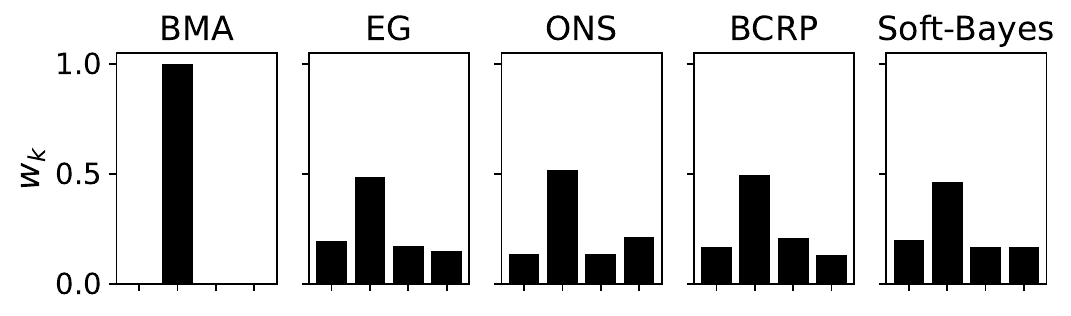}
    \caption{The final weights in the online variational inference experiment.}
    \label{fig:weights_bong}
\end{figure*}

\begin{figure*}
    \centering
    \includegraphics[width=0.7\textwidth]{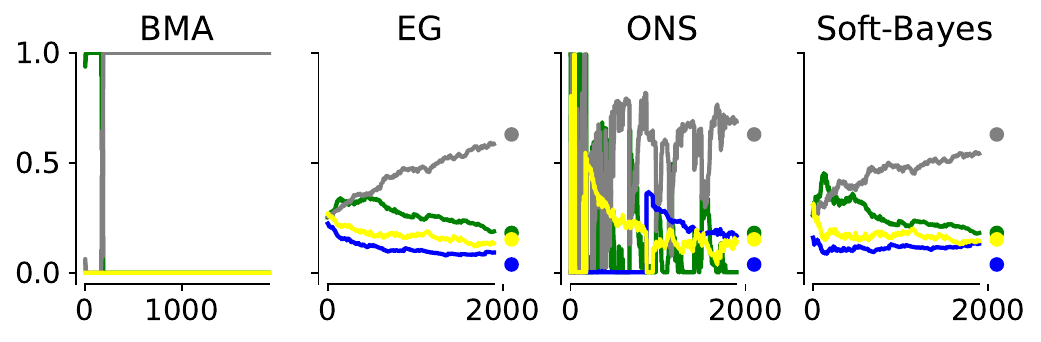}
    \caption{The evolution of the weight vector $\V{w}_t$ as a function of $t$ in the online variational inference experiment. Results are shown for the last trial. Dots on the right side of a plot denote the final weights of the BCRP.}
    \label{fig:weights_bong_evolution}
\end{figure*}

\FloatBarrier
\section{Details and Weights for the Forecasting Experiment} \label{app:more_garch}

In this appendix, we provide more details on the experimental setup and results for the forecasting experiment in \cref{sec:exp_forecasting}.

\subsection{Experimental Details}
In the forecasting experiment, we use real data corresponding to the daily returns of the S\&P 500 index from 2015 through 2020, consisting of 1257 unique observations. In econometrics, such data are often modeled using Generalized Autoregressive Conditional Heteroskedasticity (GARCH) approaches, which provide probabilistic predictions for time series. We refer the reader to \cite{engle2001garch} for a more detailed introduction to GARCH models.

\paragraph{Specification of Models} 
In \cite{geweke2011optimal,geweke2012prediction}, the application of GARCH models to stock index data is also used to test the ensembling of probabilistic models. In GARCH, the conditional variance of the process is modeled as an autoregressive process depending on the lagged values of the conditional variance and the squared ``innovations.'' For instance, the GARCH$(1,1)$ model with Gaussian innovations is given by:
\begin{align}
    y_{t} &= \sigma_t\epsilon_t, \enskip \epsilon_t \sim \mathcal{N}(0,1) \\
    \sigma^2_{t+1} &= \alpha_0 + \alpha_1\epsilon^2_t + \beta\sigma^2_{t}, \label{eq_condvar_garch}
\end{align}
where $\epsilon_t$ are the innovations and $\sigma_t$ is the conditional variance of the process. For brevity, we only discuss the GARCH$(1,1)$ with Gaussian innovations. The variants differ in the assumed process for the conditional variance.

We assign truncated Gaussian prior densities to $\boldsymbol{\theta}_t = [\alpha_0,\alpha_1, \beta, \sigma_t]$, and we are interested in sequential estimation of the posterior distribution $p(\boldsymbol{\theta}_t|y_{1:t})$. Sequential Monte Carlo (SMC) algorithms are very well suited for this task \cite{doucet2001introduction}. An SMC algorithm recursively computes a particle approximation of the posterior,
\begin{align}
    p(\boldsymbol{\theta}_t|y_{1:t}) \approx \sum_{i} \rho^{(i)}_t \delta(\boldsymbol{\theta}^{(i)}_t - \boldsymbol{\theta}_t),
\end{align}
where $\boldsymbol{\theta}^{(i)}_t$ and $\rho_t^{(i)}$ denote, respectively, the particles and weights. After the arrival of $y_{t+1}$, the weights are recomputed, and the particles are propagated to form the particle approximation of the posterior at $t+1$. Using this particle approximation, we can obtain an approximation of the posterior predictive density,
\begin{align}
    p(y_{t+1}|y_{1:t}) \approx \sum_{i}\rho_t^{(i)}p(y_{t+1}|y_{1:t},\boldsymbol{\theta}_{t+1}^{(i)}),
\end{align}
where $\boldsymbol{\theta}_{t+1}^{(i)}$ only differs from $\boldsymbol{\theta}^{(i)}_{t}$ in the component $\sigma^{(i)}_{t+1}$, which is obtained by substituting the previous set of particles in Eq. \eqref{eq_condvar_garch}.

We create ensembles using several different variants of the GARCH model, namely GARCHt (Student-t innovations),
GARCHNormal (Normal innovations),
GJR-GARCHNormal (GJR-GARCH with Normal innovations), and 
EGARCHNormal (EGARCH with Normal innovations) \cite{engle2001garch}.

\paragraph{Inference and Hyperparameter Specification} With each variant, we sample two sets of hyperparameters and perform online posterior inference using an SMC algorithm with 1000 particles and 5 Markov Chain Monte Carlo (MCMC) rejuvenation steps. At each iteration, the ensemble weights are used to evaluate, 
\begin{align}
    \sum_{k=1}^8 w_{k,t} p_k(y_{t+1}|y_t) \approx \sum_{k=1}^8 w_{k,t}\sum_{i}\rho_t^{(i)}p_k(y_{t+1}|y_{1:t},\boldsymbol{\theta}_{t+1}^{(i)}),,
\end{align}
where $w_{k,t}$ is the weight of model $k$ obtained at time $t$ using the approaches discussed in the paper. We ran 10 independent simulations of this experiment, considering 10 different random seeds.

\subsection{Weight Evolutions and Final Weights}
We visualize the resulting weights in \cref{fig:weights_garch} and \cref{fig:weights_garch_evolution}. 

\begin{figure*}[th]
    \centering
    \includegraphics[width=0.7\textwidth]{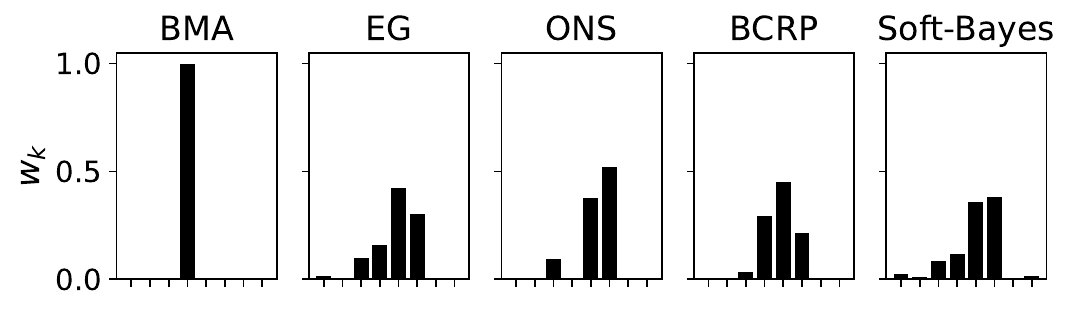}
    \caption{The final weights in the forecasting experiment.}
    \label{fig:weights_garch}
\end{figure*}

\begin{figure*}[th]
    \centering
    \includegraphics[width=0.7\textwidth]{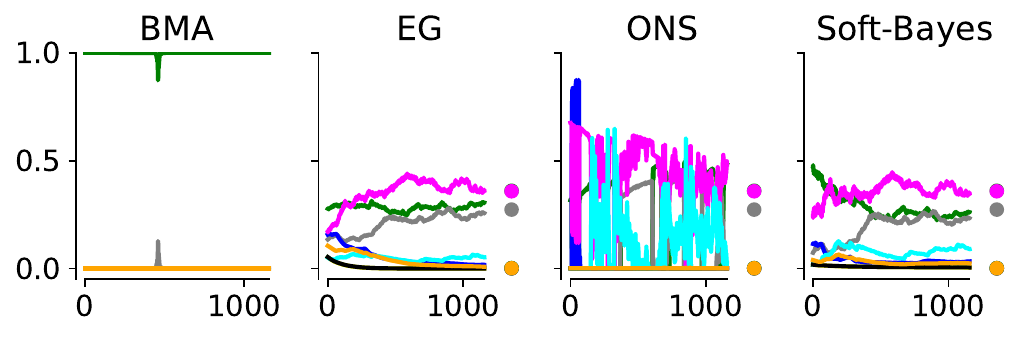}
    \caption{The evolution of the weight vector $\V{w}_t$ as a function of $t$ in the forecasting experiment. Results are shown for the last trial. Dots on the right side of a plot denote the final weights of the BCRP.}
    \label{fig:weights_garch_evolution}
\end{figure*}

\FloatBarrier
\section{Details and Weights for the Non-Stationary Environment Experiment} \label{app:more_doebe}
In our experiment in non-stationary environments, we use the real or semi-real datasets in \citet{waxman2024dynamic}. This comprises the Elevators dataset \citep{elevators}, which originates from tuning elevators on an aircraft; the SARCOS dataset \citep{rasmussen_and_williams}, which is simulated data corresponding to an inverse kinematics problem on a robotic arm; the Kuka \#1 dataset \citep{meier2014incremental}, which is real data from a task similar to SARCOS; CaData \citep{pace1997sparse} which has housing data, and CPU Small \citep{delve}, which includes various performance properties in a database of CPUs. 

As mentioned in the main text, the models considered all belong to the {\it dynamic online ensembles of basis expansions} (DOEBE) family, which perform online GP regression using Kalman filtering and the random Fourier feature approximation \cite{lazaro2010sparse}. \citet{waxman2024dynamic} find that adding a random walk to the model parameters is important in capturing ``dynamic'' (i.e., nonstationary) behavior when applying approximate GPs to several real-world datasets, which can be related to ``back-to-prior forgetting'' \citep{van2012kernel}.

We create ensembles using an RFF GP with $100$ features, trained on the first 1000 data points using the marginal likelihood. We then use different values of the random walk scale ($\sigma_\text{rw} \in \{10^{-4}, 10^{-3}, 10^{-2}, 10^{-1}, 10^{0}\}$) and ensemble them using O-BMA and various OBS algorithms.

In the interest of space, we only report the weight evolutions for SARCOS, Kuka \#1, and Elevators in \cref{fig:weights_doebe_evolution}. The other dataset results are qualitatively similar to Elevators.

\begin{figure*}[t]
     \centering
          \begin{subfigure}[b]{0.3\textwidth}
         \centering
         \includegraphics[width=\textwidth]{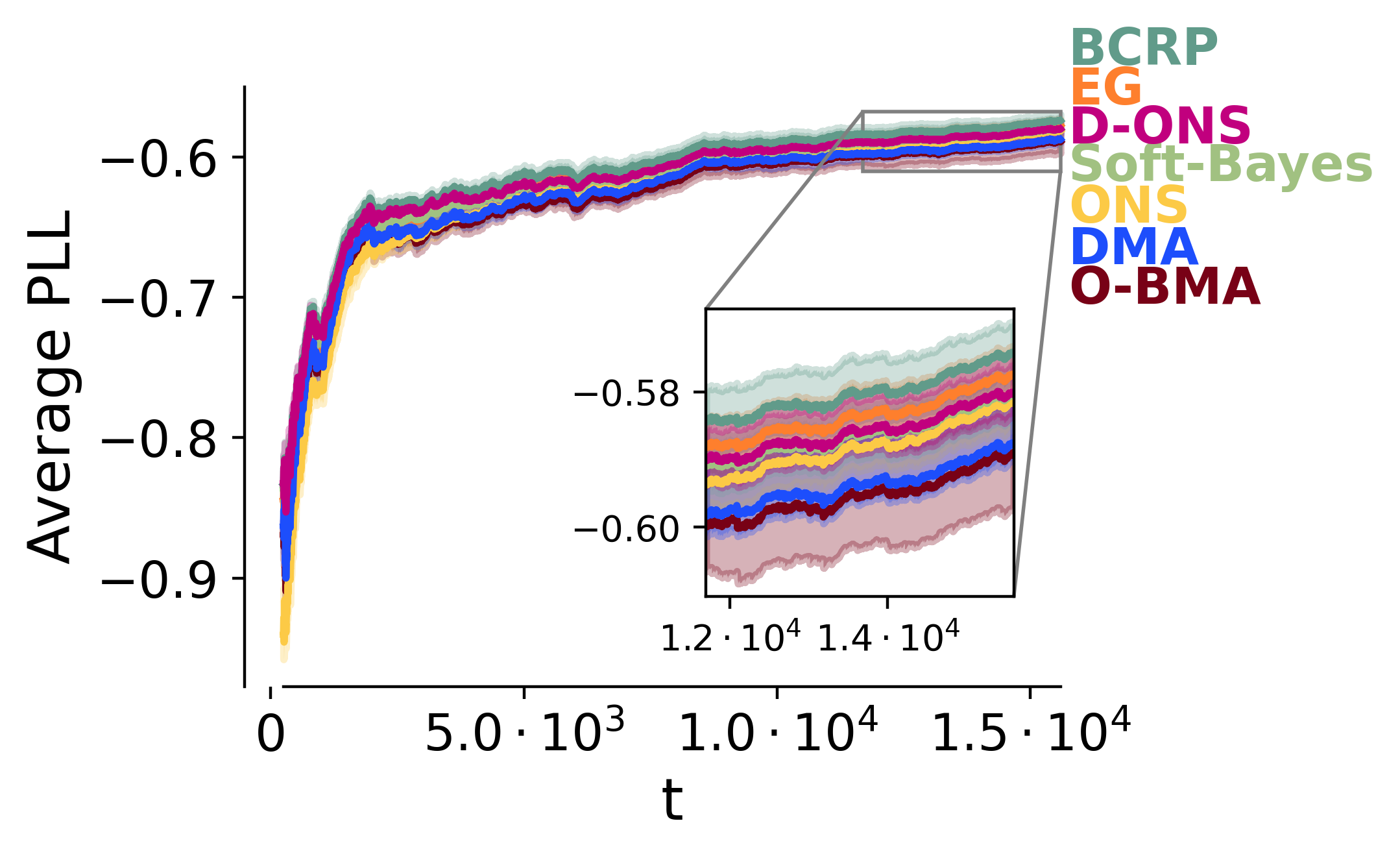}
         \caption{Elevators}
         \label{sf:Elevators_NLL}
     \end{subfigure}
     \begin{subfigure}[b]{0.3\textwidth}
         \centering
         \includegraphics[width=\textwidth]{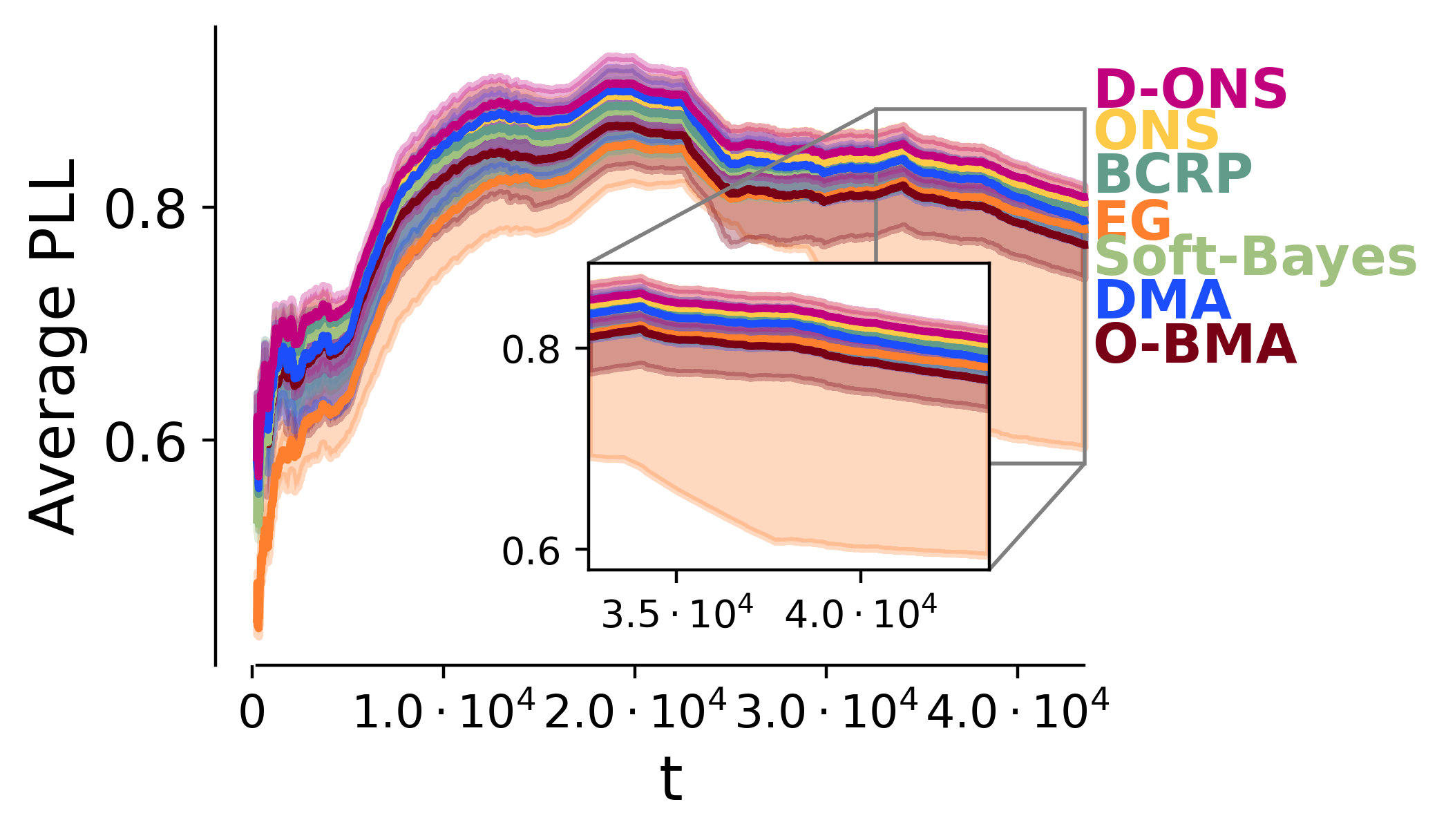}
         \caption{SARCOS}
         \label{sf:SARCOS_NLL}
     \end{subfigure}
     \begin{subfigure}[b]{0.3\textwidth}
         \centering
         \includegraphics[width=\textwidth]{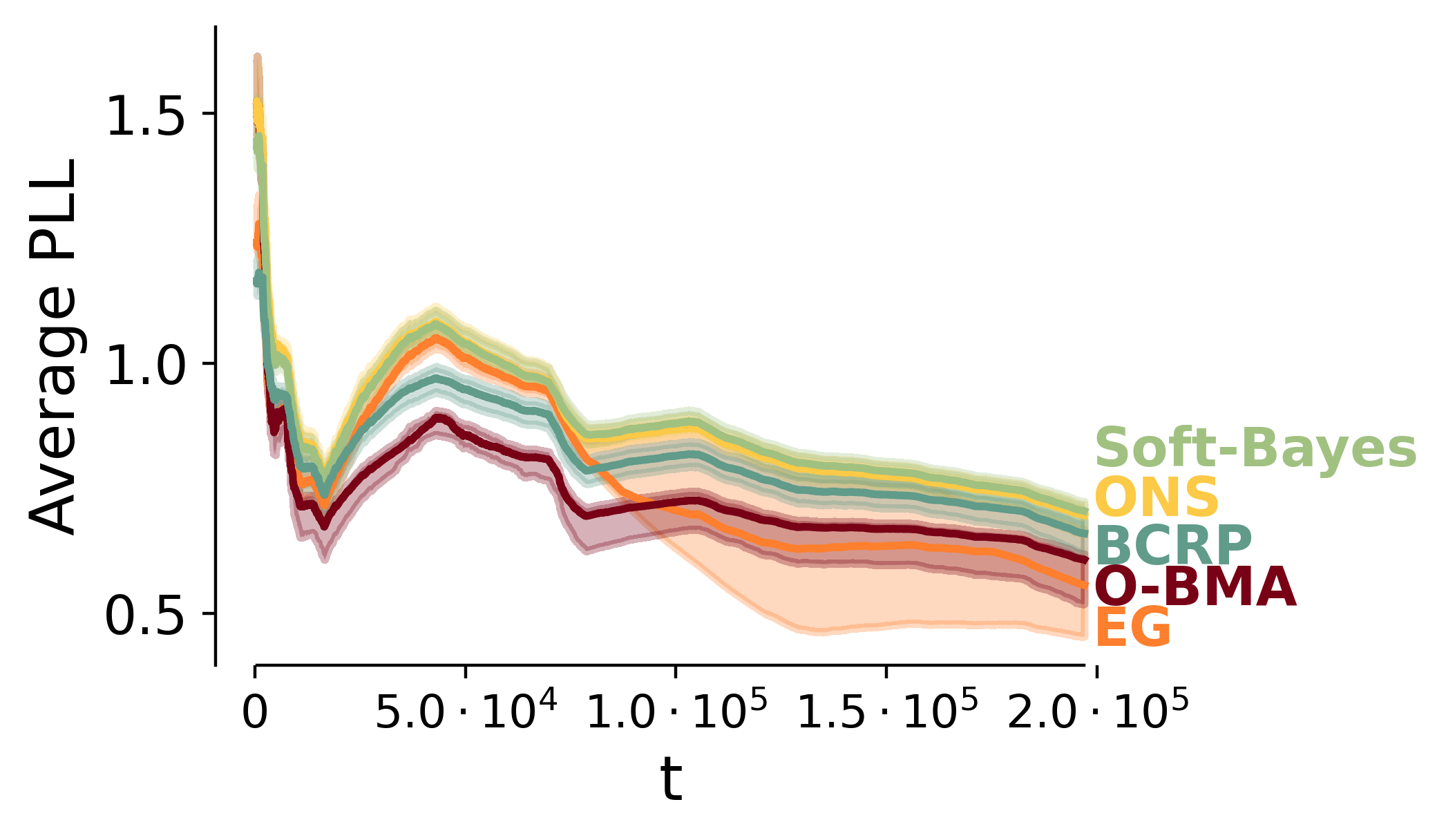}
         \caption{Kuka \#1}
         \label{sf:kuka_NLL}
     \end{subfigure}\\
    \begin{subfigure}[b]{0.3\textwidth}
         \centering
         \includegraphics[width=\textwidth]{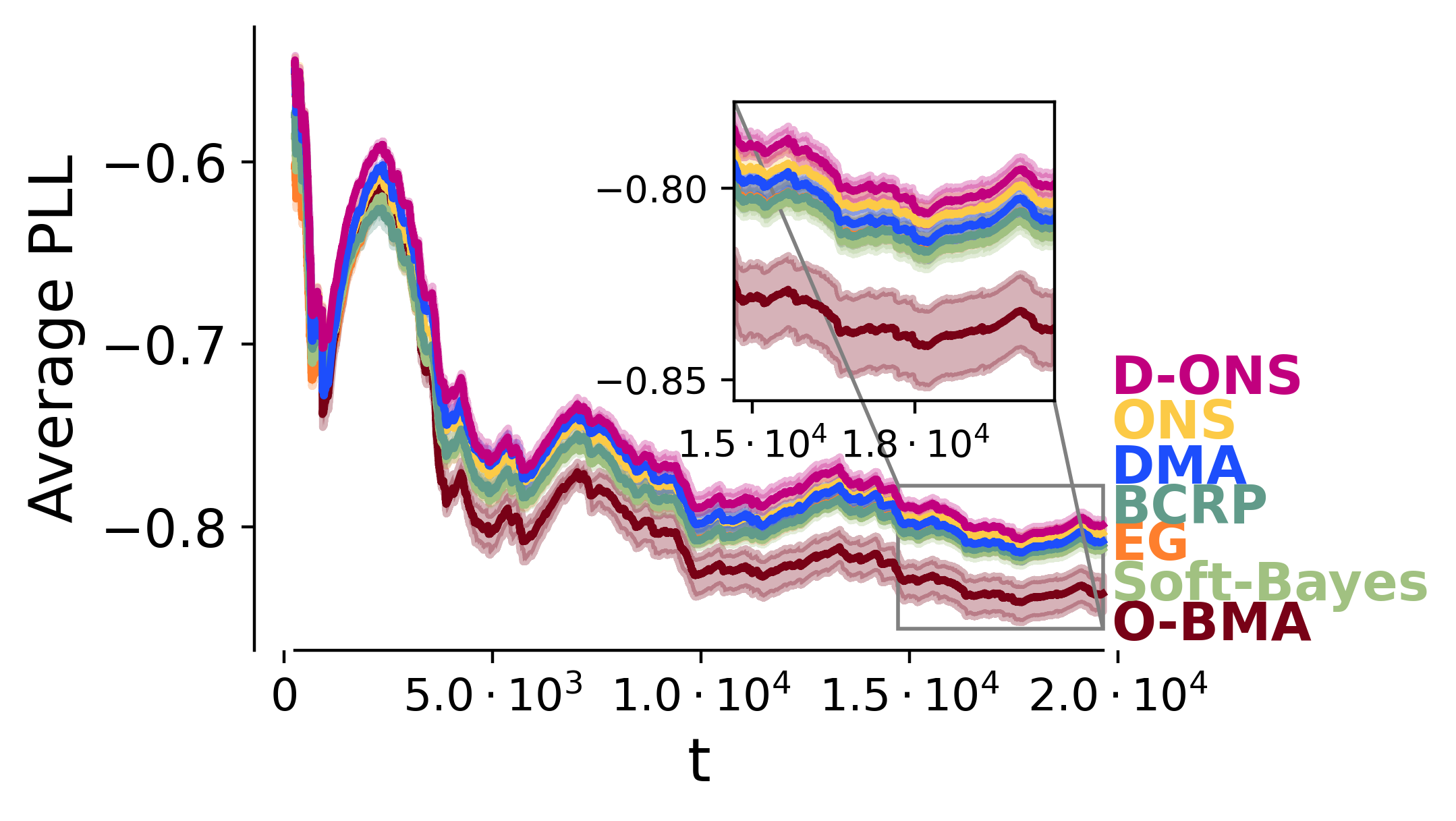}
         \caption{CaData}
         \label{sf:CaData_NLL}
     \end{subfigure}
    \begin{subfigure}[b]{0.3\textwidth}
         \centering
         \includegraphics[width=\textwidth]{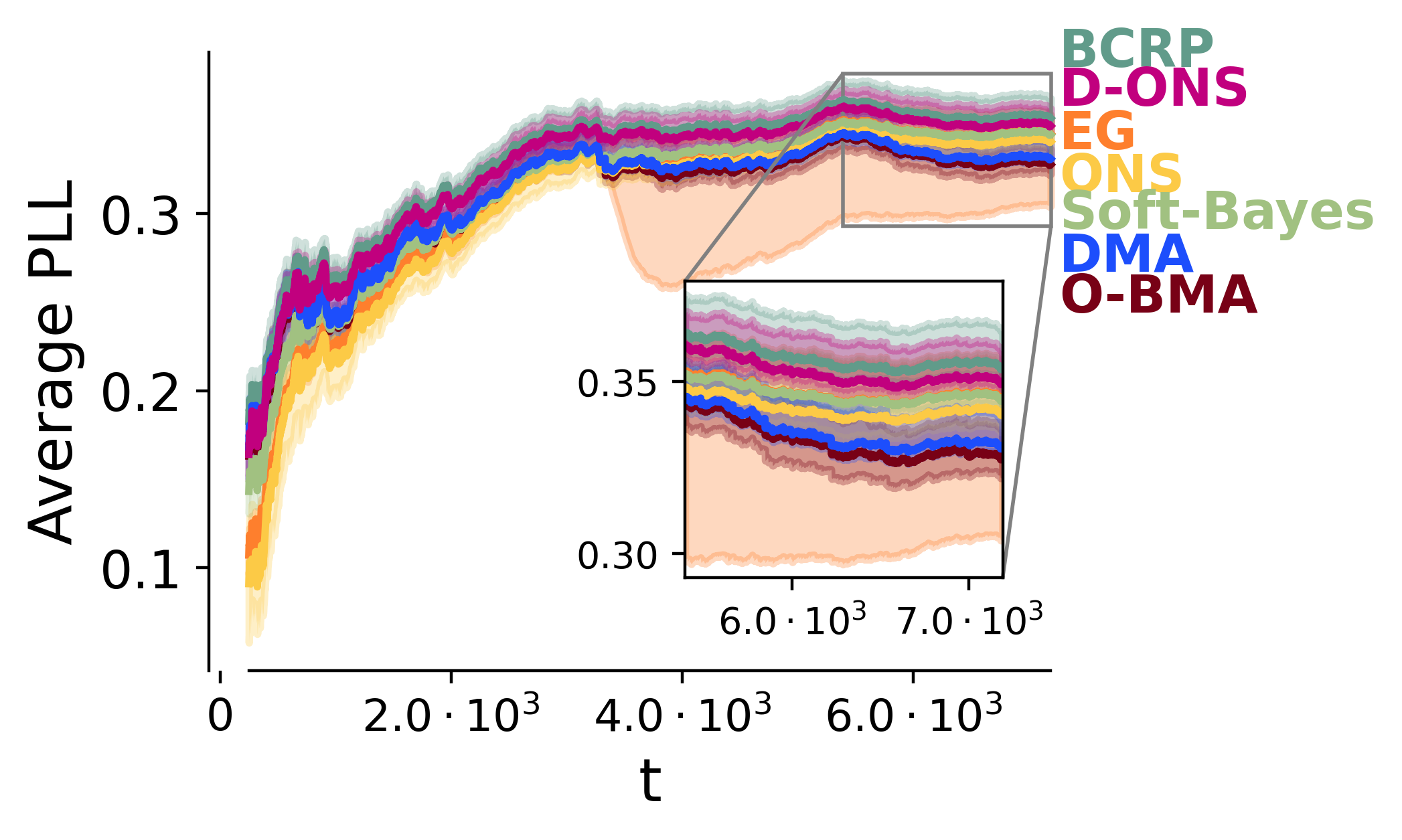}
         \caption{CPU Small}
         \label{sf:CPUSmall_NLL}
     \end{subfigure}
        \caption{The average predictive log-likelihood (higher is better) in the non-stationary experiments. Lines denote the median and shaded area represents the 10th to 90th percentiles over 10 trials. The first 250 samples are suppressed for readability.}
        \label{fig:doebe_results}
\end{figure*}

\begin{figure*}[t]
     \centering
          \begin{subfigure}[b]{0.7\textwidth}
         \centering
         \includegraphics[width=\textwidth]{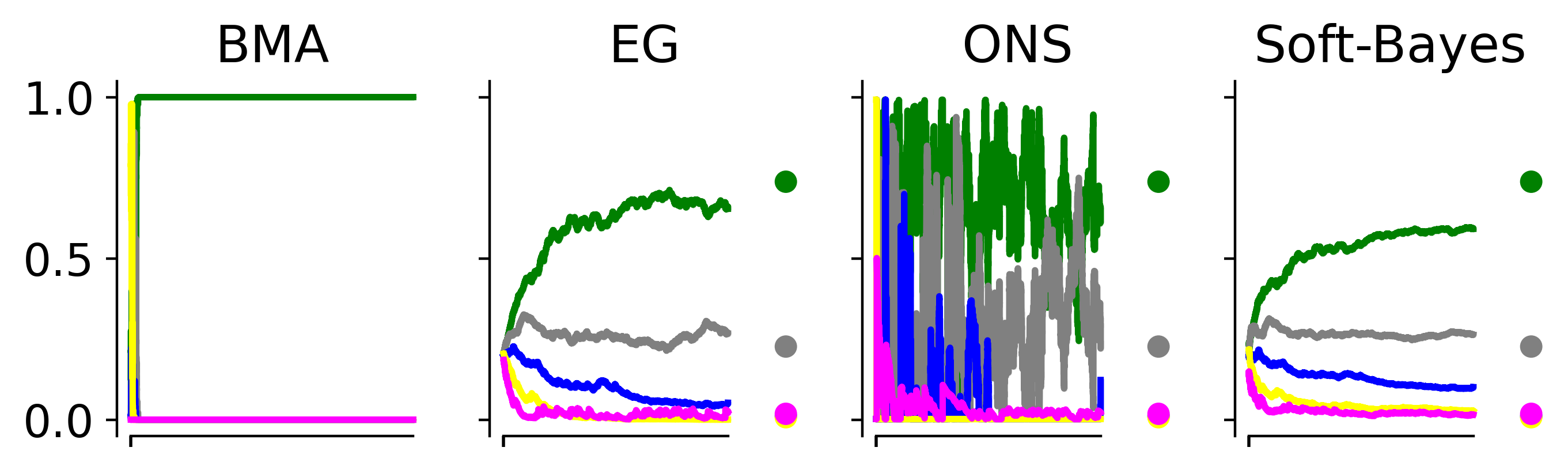}
         \caption{Elevators}
         \label{sf:Elevators_weights}
     \end{subfigure}\\
     \begin{subfigure}[b]{0.7\textwidth}
         \centering
         \includegraphics[width=\textwidth]{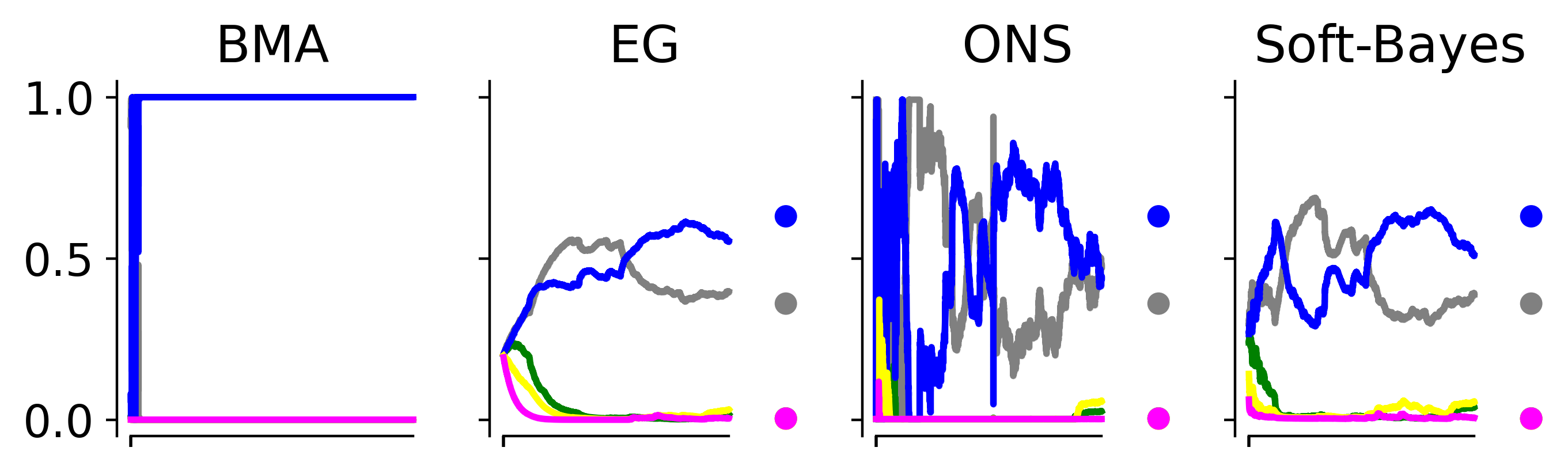}
         \caption{SARCOS}
         \label{sf:SARCOS_weights}
     \end{subfigure}\\
     \begin{subfigure}[b]{0.7\textwidth}
         \centering
         \includegraphics[width=\textwidth]{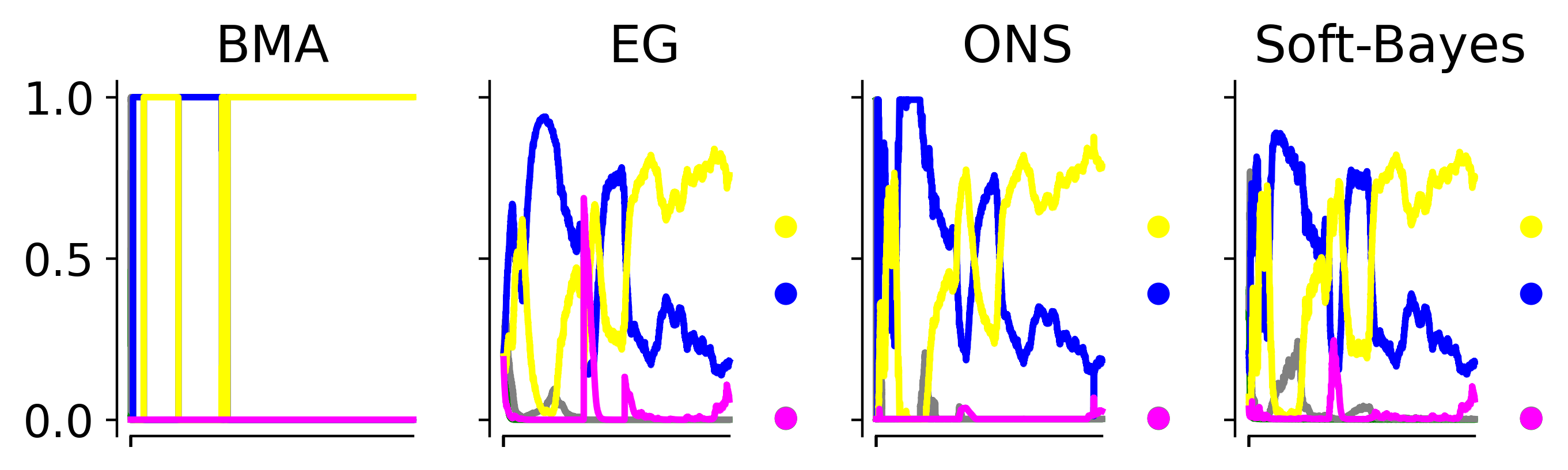}
         \caption{Kuka \#1}
         \label{sf:kuka_weights}
     \end{subfigure}
        \caption{The evolution of the weight vector $\V{w}_t$ as a function of $t$ in the non-stationary environment experiment in the Elevators, SARCOS, and Kuka \#1 datasets. Results are shown for the first trial. Dots on the right side of a plot denote the final weights of the BCRP.}
        \label{fig:weights_doebe_evolution}
\end{figure*}

\subsection{Cumulative Reward Plots}
We show the cumulative reward plots (i.e., the average PLL as a function of $t$) in \cref{fig:doebe_results}.

\section{Sensitivity to the Learning Rate} \label{app:learning_rate}

The learning rate marks an important hyperparameter in online convex optimization, with regret bounds often being quite sensitive to the learning rate. In this section, we investigate the practical effects of the learning rate parameter in EG and ONS.

For the experiment using the BONG (\cref{sec:exp_bnn} \& \cref{app:more_mnist}), we performed additional experiments using differing hyperparameters in OCO algorithms. Specifically, we vary $\eta$ in EG and $\beta$ in ONS for values $10^{-k}$ and $k\in \{0, 1, 2, 3, 4\}$. Results can be found in \cref{tab:bong_lr_ablation}. We find that the results are not particularly sensitive to the learning rate, with all results except for EG with a high learning rate ($10^{0}$) and ONS with a small learning rate ($10^{-4}$) outperforming O-BMA.

\begin{table}[ht]
    \centering
    \begin{tabular}{llc}
        \toprule
        \textbf{Method} & \textbf{Learning Rate} & \textbf{Average PLL} $\pm$ \textbf{Std. Dev.} \\
        \midrule
        EG ($\eta$) & $10^{0}$ & $-1.012 \pm 0.04$ \\
           & $10^{-1}$ & $-0.850 \pm 0.02$ \\
           & $10^{-2}$ & $-0.850 \pm 0.02$ \\
           & $10^{-3}$ & $-0.862 \pm 0.02$ \\
           & $10^{-4}$ & $-0.865 \pm 0.02$ \\
        \midrule
        ONS ($\beta) $& $10^{0}$ & $-0.878 \pm 0.03$ \\
           & $10^{-1}$ & $-0.859 \pm 0.03$ \\
           & $10^{-2}$ & $-0.860 \pm 0.02$ \\
           & $10^{-3}$ & $-0.900 \pm 0.03$ \\
           & $10^{-4}$ & $-0.948 \pm 0.03$ \\
        \midrule
        Soft-Bayes & -- & $-0.847 \pm 0.02$ \\
        \midrule
        O-BMA & -- & $-0.932 \pm 0.02$ \\
        BCRP & -- & $-0.843 \pm 0.02$ \\
        \bottomrule
    \end{tabular}
    \caption{Comparison of the average predictive log likelihood value (higher is better) for the BONG experiment (\cref{sec:exp_bnn}) while varying the learning rates of EG and ONS. Results are reported as the median $\pm$ 1 standard deviation across five trials.}
    \label{tab:bong_lr_ablation}
\end{table}

Results are similar in the forecasting example (\cref{sec:exp_forecasting} \& \cref{app:more_garch}), where the results in \cref{tab:garch_lr_ablation} suggest largely comparable performance across learning rate values. The performance of EG degrades at very large or very small learning rates, but ONS seems fairly stable in this setting.

\begin{table}[ht]
    \centering
    \begin{tabular}{llc}
        \toprule
        \textbf{Method} & \textbf{Learning Rate} & \textbf{Average PLL} $\pm$ \textbf{Std. Dev.} \\
        \midrule
        EG & $10^{0}$ & $2.776 \pm 0.23$ \\
           & $10^{-1}$ & $3.188 \pm 0.13$ \\
           & $10^{-2}$ & $3.249 \pm 0.13$ \\
           & $10^{-3}$ & $2.978 \pm 0.14$ \\
           & $10^{-4}$ & $2.858 \pm 0.17$ \\
        \midrule
        ONS & $10^{0}$ & $3.256 \pm 0.12$ \\
           & $10^{-1}$ & $3.330 \pm 0.13$ \\
           & $10^{-2}$ & $3.300 \pm 0.13$ \\
           & $10^{-3}$ & $3.229 \pm 0.10$ \\
           & $10^{-4}$ & $3.159 \pm 0.12$ \\
        \midrule
        Soft-Bayes & -- & $3.296 \pm 0.13$ \\
        \midrule
        O-BMA & -- & $3.040 \pm 0.24$ \\
        BCRP & -- & $3.296 \pm 0.12$ \\
        \bottomrule
    \end{tabular}
    \caption{Comparison of the average predictive log likelihood value (higher is better) for the GARCH experiment (\cref{sec:exp_forecasting}) while varying the learning rates of EG and ONS. Results are reported as the median $\pm$ 1 standard deviation across ten trials.}
    \label{tab:garch_lr_ablation}
\end{table}

\section{Dynamic Model Averaging With Differing Forgetting Factors}

In our experiments in \cref{sec:online_nonstationary}, we used a default ``forgetting factor'' for DMA, as recommended in \citet{raftery2010online}. This is similar to our other experiments, in which we leave OBS hyperparameters at reasonable defaults and do not otherwise tune them. Nevertheless, to further understand the impacts of this hyperparameter, we repeated our experiments with more aggressive and more conservative values of the forgetting factor. We report results in \cref{tab:dma_with_other_lr}. 

Overall, we find that DMA with extremely low forgetting factors performs well in highly nonstationary environments (Kuka \#1, CaData), but it is still comparable to D-ONS, which has more robust guarantees regarding predictive performance. Moreover, this is without tuning the hyperparameters of D-ONS. In other datasets tested, OBS methods are still superior to DMA (CPU Small, SARCOS, Elevators). We additionally remark that DMA has much similarity with a sliding window approach and loses many of the statistical properties of O-BMA; on the other hand, adaptive/dynamic OCO algorithms D-ONS still come with robust theoretical guarantees.

\begin{table}[h]
\centering
\begin{tabular}{lccccc}
\toprule
Method & Elevators & SARCOS & Kuka \#1 & CaData & CPU Small \\
\midrule
O-BMA      & -0.59 & 0.77 & 0.61 & -0.84 & 0.33 \\
DMA        & -0.59 & 0.79 & 0.69 & -0.81 & 0.33 \\
DMA-0.9    & -0.60 & 0.80 & \textbf{0.75} & -0.80 & 0.33 \\
DMA-0.95   & -0.59 & 0.80 & 0.73 & -0.80 & 0.33 \\
EG         & -0.58 & 0.78 & 0.56 & -0.81 & 0.35 \\
Soft-Bayes & -0.58 & 0.80 & 0.70 & -0.81 & 0.34 \\
BCRP       & \textbf{-0.57} & 0.80 & 0.66 & -0.81 & \textbf{0.35} \\
ONS        & -0.58 & 0.80 & 0.70 & -0.80 & 0.34 \\
D-ONS      & -0.58 & \textbf{0.81} & 0.73 & -0.80 & 0.35 \\
\bottomrule
\end{tabular}
\caption{Comparison of methods on non-stationary datasets (c.f. \cref{tab:nonstationary_results}), including DMA with additional forgetting factors.} \label{tab:dma_with_other_lr}
\end{table}

\end{document}